%% file: main.tex
\documentclass[journal,twoside,web]{ieeecolor}

\usepackage{generic}

\makeatletter
\let\originaltitlepagestyle\ps@titlepagestyle
\def\ps@titlepagestyle{%
    \originaltitlepagestyle
    \def\@oddhead{}%
    \def\@evenhead{}%
}
\makeatother

\usepackage{cite}
\usepackage{amsmath,amssymb,amsfonts}
\usepackage{graphicx}
\usepackage{xurl}
\usepackage{hyperref}
\hypersetup{hidelinks=true}
\usepackage{textcomp}

\usepackage[utf8]{inputenc}
\usepackage[T1]{fontenc}
\usepackage{lmodern}
\usepackage{graphicx}
\usepackage{glossaries}
\usepackage{amsmath,amssymb,amsfonts}
\usepackage{mathtools}
\usepackage{bm}
\usepackage{bbm}
\usepackage{hyperref}
\usepackage{microtype}
\usepackage{cleveref}
\usepackage{xparse}
\usepackage{thmtools}
\usepackage{subcaption}
\usepackage{thm-restate}
\usepackage{comment}
\usepackage{algorithm}
\usepackage{caption}
\usepackage{algpseudocode}

\def\BibTeX{{\rm B\kern-.05em{\sc i\kern-.025em b}\kern-.08em
    T\kern-.1667em\lower.7ex\hbox{E}\kern-.125emX}}
\input{tex/tex_files/macros}
\input{tex/tex_files/metadata}

\begin{document}
\title{\customtitle}
\author{Raham A. Butt, Marco Romanelli, and Roche C. de Guzman
    \thanks{
        This work was supported in part by NSF Award No. 2535091 through UNC-Chapel Hill subaward No. 5136002 to Hofstra University.  R.C. de Guzman was supported in part by NIH/NIBIB Award No. R25EB035522.  The authors thank the FRINK/Proto-OKN teams at RENCI and SDSC for providing infrastructure and technical support for the Medical Device Knowledge Graph and acknowledge the use of Hofstra’s StarHPC cluster, supported by NSF Award No. 2320735.
    }
    \thanks{R. A. Butt is with the Department of Engineering, Hofstra University, Hempstead, NY 11549 USA (e-mail: raham.a.butt@hofstra.edu).}
    \thanks{M. Romanelli is with the Department of Computer Science, Hofstra University, Hempstead, NY 11549 USA (e-mail: marco.romanelli@hofstra.edu).}
    \thanks{R. C. de Guzman is with the Department of Engineering, Hofstra University, Hempstead, NY 11549 USA (e-mail: RocheRey.C.DeGuzman@hofstra.edu).}
}

\maketitle

\begin{abstract}
    \input{tex/tex_files/sections/abstract}

\end{abstract}


\section{Introduction}
\label{sec:introduction}
\input{tex/tex_files/sections/intro.tex}

\section{Related Work}
\label{sec:related-work}
\input{tex/tex_files/sections/related_work.tex}

\section{Framework}
\label{sec:framework}
\input{tex/tex_files/sections/framework.tex}

\section{Experiments}
\label{sec:experiments}
\input{tex/tex_files/sections/experiments.tex}

\section{Conclusion}
\label{sec:conclusion}
\input{tex/tex_files/sections/conclusion.tex}

\section*{References}
\bibliographystyle{IEEEtran}
\bibliography{tex/bibliography}

\appendix
\section{Appendix}
\label{sec:appendix}
\input{tex/tex_files/sections/appendix.tex}

\end{document}

%% file: tex/tex_files/macros.tex
\newcommand{\R}{\mathbb{R}}

\newcommand{\calY}{\mathcal{Y}}
\newcommand{\calX}{\mathcal{X}}
\newcommand{\calD}{\mathcal{D}}
\newcommand{\x}{\mathbf{x}}
\newcommand{\card}[1]{\left| #1 \right|}
\newcommand{\h}[1]{H_{#1}}
\newcommand{\pXY}{p_{XY}}
\newcommand{\PYXhat}{P_{\widehat{Y}|{X}}}

\newcommand{\phat}{\widehat{p}}
\newcommand{\PYX}{P_{Y|X}}
\newcommand{\fun}[1]{\text{#1}}
\newcommand{\argmax}{\operatornamewithlimits{argmax}}
\NewDocumentCommand{\W}{g}{%
    \IfNoValueTF{#1}
    {W\left(p,\phat\right)}
    {W_{#1}\left(p,\phat\right)}%
}
\newcommand{\costmult}{k}
\newcommand{\placeholderdist}{q}

\crefname{algocf}{Algorithm}{Algorithms}   
\Crefname{algocf}{Algorithm}{Algorithms}
\newacronym{pdf}{pdf}{probability density function}
\newacronym{kkt}{KKT}{Karush-Kuhn-Tucker}
\newacronym{cmp}{CMP}{Constrained Mathematical Programming}
\newacronym{ml}{ML}{machine learning}
\newacronym{ai}{AI}{artificial intelligence}
\newacronym{mlp}{MLP}{multi-layer perceptron}
\newacronym{protookn}{Proto-OKN}{Prototype Open Knowledge Network}
\newacronym{okn}{OKN}{Open Knowledge Network}
\newacronym{kg}{KG}{Knowledge Graph}

%% file: tex/tex_files/metadata.tex
\newcommand{\customtitle}{
    Epistemic Uncertainty-Aware Defect Detection for Quality Control in Medical Device Manufacturing
}

%% file: tex/tex_files/sections/abstract.tex
\textit{Objective:} We investigate whether accounting for epistemic uncertainty can improve the reliability of automated defect detection in medical device manufacturing.

\textit{Methods:} We consider a \acrlong*{ml} framework that operates on heterogeneous manufacturing and device-report data 
represented with \acrlongpl*{kg}. To mitigate errors arising from uncertainty in the decision model, we analyze a principled rejection strategy to abstain from predictions whose estimated epistemic uncertainty exceeds a specified threshold. We evaluate the approach using standard synthetic benchmarks and real-world medical device report data.

\textit{Results:} The theoretical results establish the validity of the method characterizing the regimes 
under which it is expected to be effective. Empirically, the rejection strategy enables explicit control of coverage, that is, the proportion of samples for which the model issues predictions, while improving performance on the retained samples. On 266,170 real-world FDA MAUDE device reports, a 10\% abstention rate reduces classification error by 48\%, and more aggressive rejection (approximately 70\% coverage) yields near-perfect accuracy on the retained samples. On standard synthetic manufacturing benchmarks, abstaining on
\(
9\%\) of the decisions, our approach reduces the risk up to \(
63\%\) compared with the standard no-abstention approach.
%

\textit{Conclusions:} Abstaining from predictions with high epistemic uncertainty can provide a practical tool for controlling the reliability of \acrlong*{ml}-based defect detection, 
especially in high-stakes medical device manufacturing applications.

\textit{Significance:} Uncertainty-aware defect detection may support safer and more reliable quality assurance in medical device manufacturing by identifying cases that require additional inspection rather than issuing potentially harmful predictions.

%% file: tex/tex_files/sections/intro.tex
\IEEEPARstart{D}{efective} medical devices can cause serious patient harm, and their timely detection, both on the manufacturing line and in post-market adverse event reports, is central to device safety~\cite{fda_maude}. Such defects can arise due to various reasons, including hardware malfunctions, software bugs, or environmental factors. Detecting these anomalies promptly is essential to prevent potential harm to patients and to maintain the integrity of the medical device. The \emph{complexity} of modern medical devices, which often incorporate advanced semiconductor components, \emph{may limit the effectiveness of traditional defect detection methods} \cite{huang2026unsuperviseddefectdetectionsurgical}. As a result, there is a growing interest in leveraging machine learning techniques to enhance the detection process (cf. \cite{bonatti2025expert,
    kim2024deep,lu2019realtime,cao2024xscan})

In particular, \cite{shenoy2026uncertainty} introduces a general approach to defect–mechanism–parameter violation hypotheses for surface-mount assembly root-cause analysis from uncertain observations. Using \glspl*{kg} as the source of information, the authors propose a method to detect and explain manufacturing defects. In doing so, they leverage stage-wise three-head cascade \gls*{mlp} architecture that models aleatoric uncertainty by estimating defect and mechanism likelihoods and parameter-level risk scores, enabling probabilistic and traceable root-cause ranking. Specifically, Head 1 detects the defects or lack thereof, Head 2 identifies the underlying mechanisms, and Head 3 estimates the associated risks. The authors note that their method can be sensitive to \emph{epistemic uncertainty}, which may arise from the limited amount of information available during the training phase (cf. \cite{DBLP:conf/cvpr/Mukhoti0ATG23,lakshminarayanan2017ensembles,kendall2017uncertainties}). The distinction between aleatoric and epistemic uncertainty is well established in the \gls*{ml} literature~\cite{hullermeier2021aleatoric}, and the need to communicate and act on model uncertainty is increasingly recognized as a requirement for machine learning in health-related, high-stakes applications~\cite{kompa2021second,begoli2019need}.

Crucially, this can lead to detection failures and misclassification of defects. This is a significant concern in high-stakes applications such as medical device manufacturing, where misclassification can incur substantial costs and pose a significant risk of serious patient harm. A wrong detection performed by Head 1, for instance, can lead to a wrong alert being raised, potentially incurring substantial costs (false positives) or posing a significant risk of serious patient harm (false negatives).

\input{tex/tex_files/plots/toy_example_fig.tex}

We argue that addressing this issue is critical and, therefore, investigate the impact of epistemic uncertainty on detection failures in medical device manufacturing and examine a principled approach to mitigating its effects. Specifically, we adopt a cost-sensitive classification framework that explicitly accounts for uncertainty in the decision-making process by allowing the model to abstain from uncertain predictions. This formulation aims to improve the robustness and reliability of defect detection while accounting for the costs associated with erroneous decisions.

In our framework, the task performed by \gls*{mlp} Head 1 is treated as a critical decision point, as an incorrect prediction at this stage may compromise the reliability of subsequent decisions and risk assessment. To mitigate this risk, the model is allowed to abstain when its prediction is uncertain, thereby reducing the likelihood of misclassification and improving the overall reliability of the defect detection system. Although outside the scope of this work, abstained predictions could subsequently be subjected to further scrutiny by human experts or additional automated checks. Such a mechanism provides a more cautious approach to defect detection in high-stakes settings and could further enhance the reliability and safety of medical device manufacturing.

\subsection{Contributions}
\label{subsec:contributions}
We summarize the contributions contained in the rest of this paper as follows:
\begin{enumerate} \item We introduce a reproducible real-world benchmark from a medical device case study, publicly accessible as a \gls*{kg} through the \gls*{protookn} (cf. \cite{protookn,oknregistry}), demonstrating the practical applicability of the proposed method in a high-stakes setting. All code required to reproduce the experiments, including the procedure used to generate the benchmarks, is available at \url{https://github.com/raham24/Epistemic-Uncertainty-Aware-Defect-Detection-for-Quality-Control-in-Medical-Device-Manufacturing/tree/code-only}.
    \item
          We investigate detection failures arising from epistemic uncertainty in recent root-cause analysis approaches for semiconductor manufacturing, and extend that analysis to high-stakes domains such as defect detection in medical device manufacturing, where misclassification can lead to substantial costs and serious risks to patient safety.
    \item We analyze a simple yet effective approach to mitigate the impact of epistemic uncertainty on detection failures, leveraging a cost-sensitive classification framework that incorporates uncertainty handling into the decision-making training process, allowing the model to abstain from uncertain predictions. Our theoretical results provide a principled foundation with a mathematical-programming approach, precisely characterizing the regimes under which it is expected to be effective and complementing the intuition developed in prior work.
    \item We evaluate the proposed approach through extensive experiments on a synthetic benchmark spanning varying levels of class-posterior overlap and five random seeds. The results demonstrate consistent effectiveness as increasing overlap makes the classification task progressively more challenging, as well as robustness to random initialization.
\end{enumerate}

%% file: tex/tex_files/plots/toy_example_fig.tex
\begin{figure*}[!h]
    \centering
    \begin{subfigure}[c]{0.55\textwidth}
        \centering
        \includegraphics[width=\linewidth]{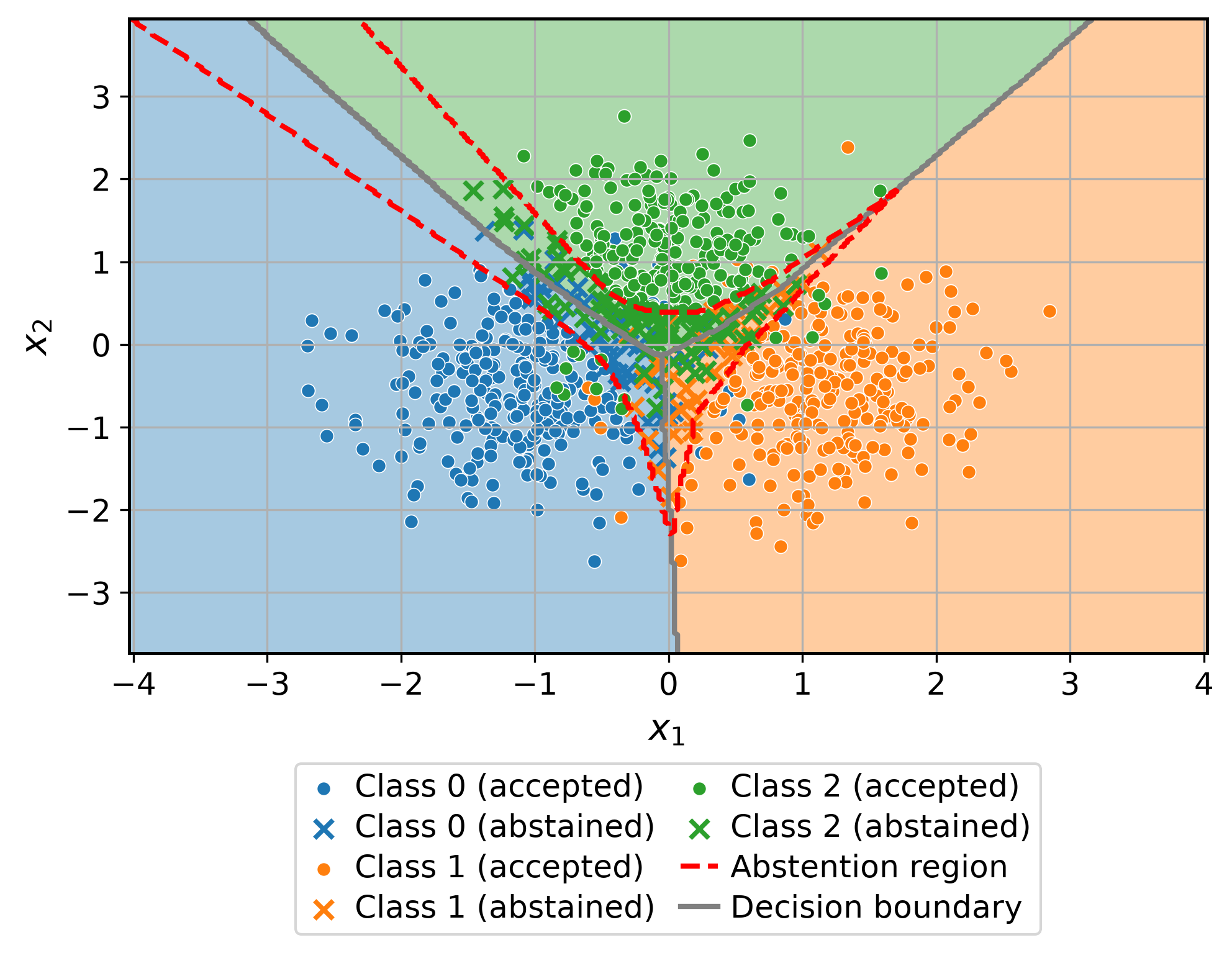}
        \caption{Decision regions and rejected points.}
        \label{fig:decision_regions}
    \end{subfigure}
    \hspace{0.01\textwidth}
    \begin{minipage}[c]{0.32\textwidth}
        \centering
        \begin{subfigure}{\linewidth}
            \centering
            \includegraphics[width=\linewidth]{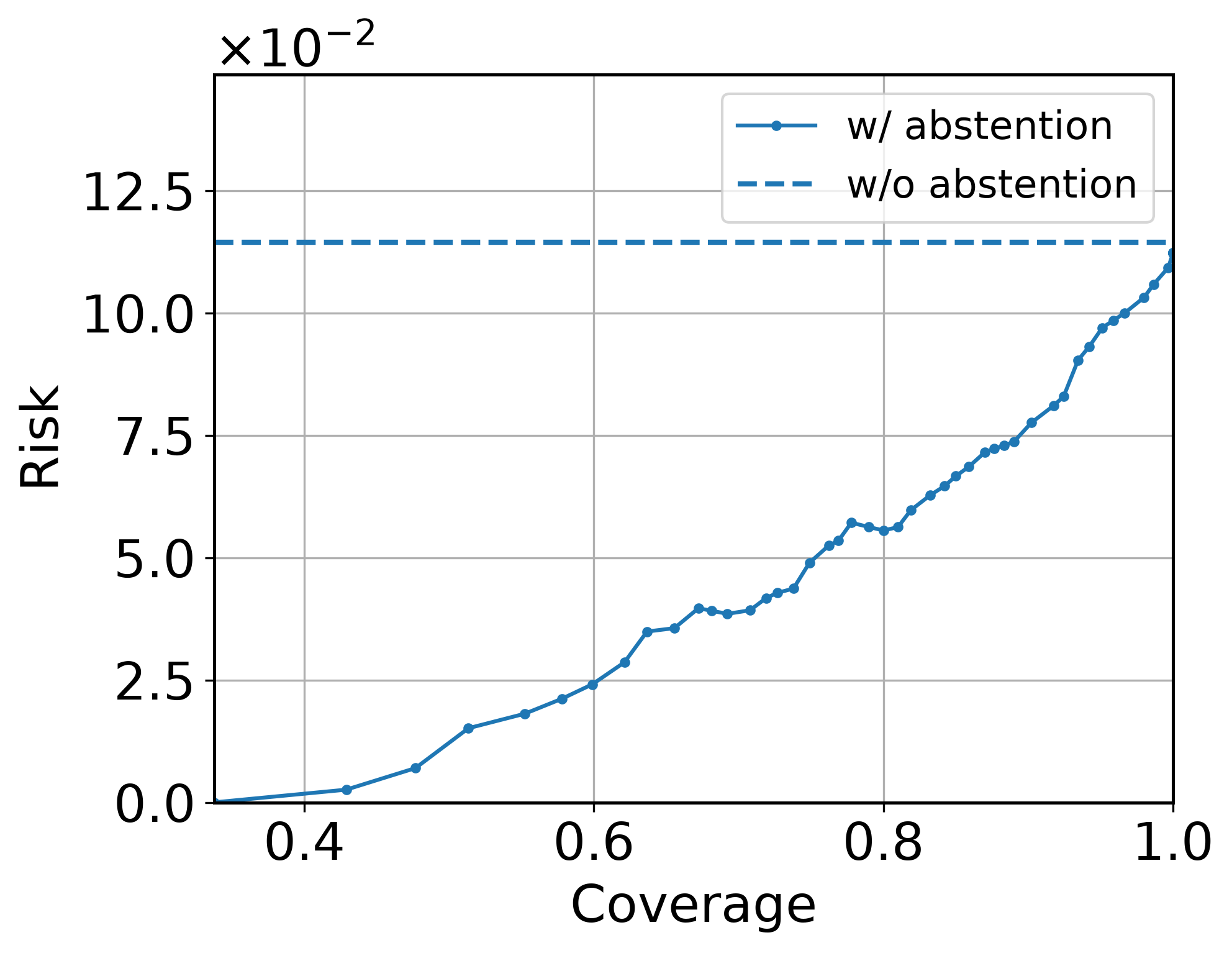}
            \caption{Selective risk curve.}
            \label{fig:risk_curve}
        \end{subfigure}
        \vspace{0.8em}
        \begin{subfigure}{\linewidth}
            \centering
            \includegraphics[width=\linewidth]{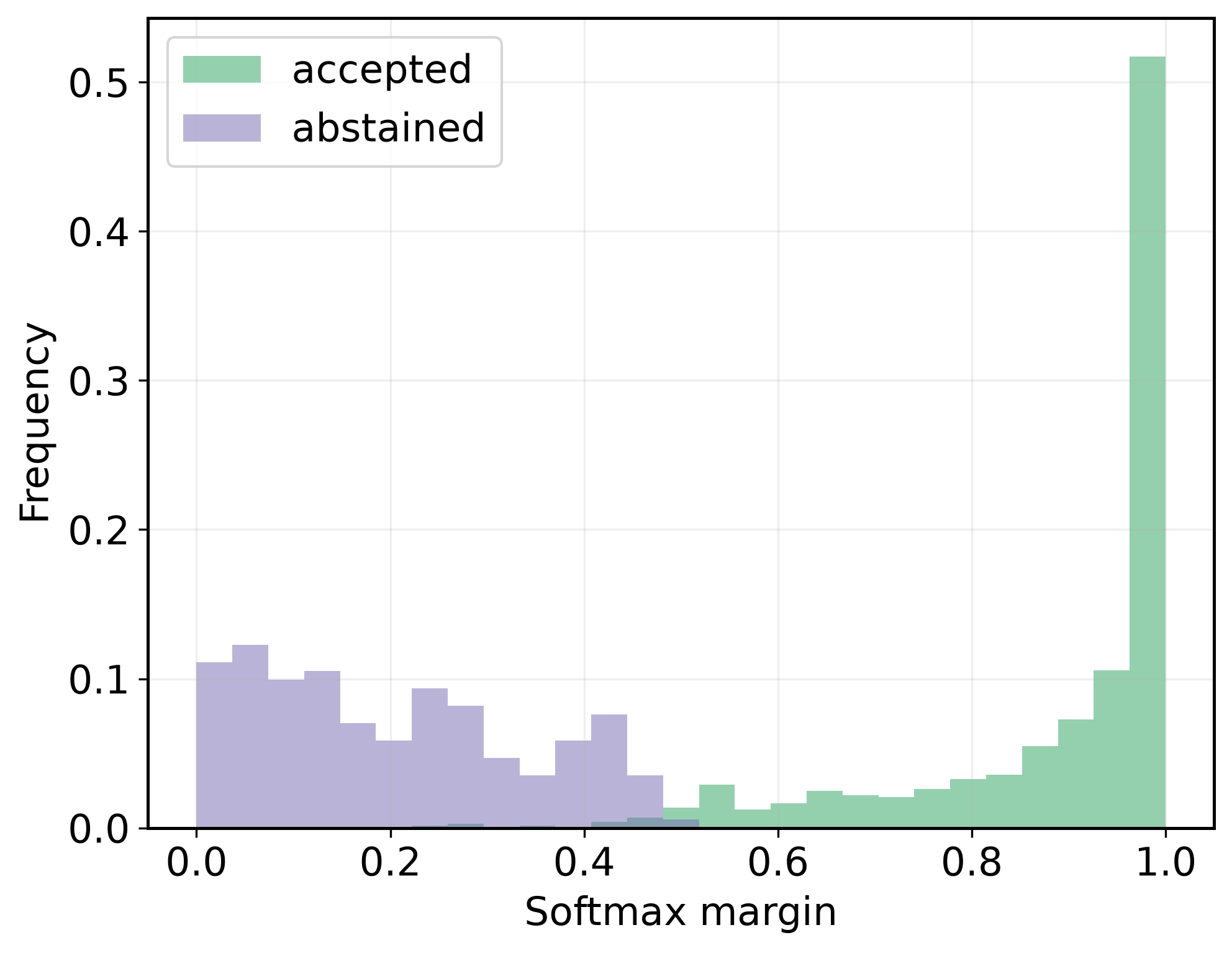}
            \caption{Classification margin histogram.}
            \label{fig:margin_histogram}
        \end{subfigure}
    \end{minipage}
    \caption{Illustrative example -- The impact of principled abstention on the decision regions (\Cref{fig:decision_regions}), 
    risk-coverage curve (\Cref{fig:risk_curve}), and classification margin histogram (\Cref{fig:margin_histogram}) for a toy dataset with 2-dimensional features and 3 balanced classes. The model rejects observations concentrated near decision boundaries, where predictions are more uncertain. Lowering the coverage
        by 
        \(19\%\) reduces the classification risk on retained predictions by 
        \(50\%\). The softmax-margin distributions further show that rejected observations generally have smaller margins than retained observations, with only limited overlap, indicating that the reservation mechanism effectively identifies uncertain predictions and trades coverage for lower selective risk. Further details are provided in \Cref{subsec:illustrative_example}.
        %
    }
    \label{fig:illustrative_example}
\end{figure*}

%% file: tex/tex_files/sections/related_work.tex
The increasing integration of digital technologies into medical devices has introduced new challenges for quality assurance. To address these challenges, \gls*{ai}-based defect detection has gained attention as an efficient alternative to conventional rule-based and manual inspection methods (cf. \cite{kim2024deep,bonatti2025expert}).

In the context of software reliability and cybersecurity, \cite{kim2024deep} proposed an AI-based software defect detection method as part of a broader framework for strengthening medical device security. Their work highlights the potential of \gls*{ai} to provide a more efficient alternative to conventional rule-based approaches for identifying software flaws that may compromise the reliability and safety of medical devices.

From a manufacturing perspective, \cite{bonatti2025expert} developed an automated expert system for detecting porosity defects in prosthetic components produced using electron beam melting. By analyzing layer-by-layer optical images acquired during manufacturing, their approach enables rapid and nondestructive identification of defects and demonstrates strong agreement with computed tomography-based inspection.

More recently, \cite{huang2026unsuperviseddefectdetectionsurgical} investigated unsupervised defect detection for surgical instruments and identified several challenges associated with applying conventional industrial defect detection methods to the medical device domain, including textured backgrounds, small defects, and domain shifts. Their proposed approach combines background masking, patch-based analysis, and domain adaptation to improve the detection of fine-grained defects. Collectively, these studies demonstrate the growing role of AI-based methods in improving the efficiency of defect detection across both the software and physical components of medical devices, while emphasizing the importance of reliable detection mechanisms in safety-critical healthcare applications.

As many modern medical devices rely on semiconductor components, advances in industrial semiconductor defect detection are also relevant to the medical device domain. Neural methods have been widely applied to defect detection, fault classification, and anomaly prediction in semiconductor manufacturing, particularly for analyzing process and inspection data. Prior work \cite{lee2017convolutional,khader2018stencil,lu2019realtime,cao2024xscan} has focused on improving detection accuracy and enabling real-time process control, including for solder paste printing, a major source of downstream defects. More recently, \cite{shenoy2026uncertainty} introduced a framework for generating structured causal explanations for root cause analysis, incorporating uncertainty to quantify the risk associated with inferred defect causes.

Crucially, these works do not explicitly account for uncertainty arising from the model's decisions themselves, i.e., epistemic uncertainty, which is a central focus of our work. The use of abstention (or rejection option) and selective classification to manage uncertainty in decision-making has been studied for decades (cf.~\cite{Flores1958TEC,Chow1970TIT,vovk2005algorithmic,angelopoulos2023conformal,bates2021distribution,elyaniv2010foundations}), providing a mechanism for avoiding low-confidence decisions and reducing incorrect classifications~\cite{PudilNBK1992ICPR}.
Several approaches have been proposed to incorporate abstention into predictive models. Some methods require dedicated model architectures~\cite{CorbiereTBCP2019NeurIPS,GeifmanE2019ICML}, while others rely on perturbed training data to improve empirical generalization~\cite{HuangZZ2020NeurIPS,FischJB2022CORR} or on specialized training procedures and ad hoc loss functions~\cite{FengAHA2022CORR}.

Alternatively, post-training approaches such as \cite{GraneseRGPP2021NeurIPS} leverage information-theoretic measures to quantify uncertainty without modifying the underlying model, providing a computationally efficient and principled solution, although their effectiveness may be limited for overconfident models (cf. \cite{GuoPSW2017ICML}).
Our work builds on the framework of \cite{liu2019deep}, which draws on the decision-theoretic principles of \cite{markowitz1952portfolio}. We extend this framework by providing additional theoretical insights, including an interpretable characterization of the abstention threshold (\Cref{theorem:kkt}) and a complete mathematical programming characterization through the associated optimality conditions (\Cref{theorem:threshold}).

Finally, the recent development of resources such as \gls*{protookn}~\cite{protookn,oknregistry}, which integrate heterogeneous data from sources including the FDA 510(k) and MAUDE databases~\cite{fda_510k,fda_maude}, provides new opportunities to investigate uncertainty-aware defect detection using real-world medical device data. \gls*{kg} are an established substrate for integrating heterogeneous biomedical data for downstream analysis~\cite{nicholson2020constructing}. On the device-safety side, the MAUDE database underpins a substantial surveillance literature, including reviews of signal-detection methodology~\cite{chung2020signal}, change-point analysis for post-market signals~\cite{xu2015changepoint}, natural-language processing of adverse-event narratives~\cite{mao2021nlp}, large-language-model classification of device-related adverse events~\cite{sohoni2026llm}, and machine-learning prediction of FDA recalls from public data~\cite{slivinskis2025recall}. To our knowledge, none of this line of work addresses epistemic uncertainty or selective prediction, which is the gap our real-data benchmark targets.

%% file: tex/tex_files/sections/framework.tex
We focus on the classification task associated with head \(\h{1}\), where errors can result in either false or missed defect alarms, increasing production costs and risk. The same analysis extends naturally to other classification heads. We therefore formulate the task for \(\h{1}\) as a multi-class classification problem.

Let \(\calX\subseteq\R^d\) be the (possibly continuous) feature space and let  \(\calY=\{1,\dots,C\}\) be the label space for the classification task, thus \(\card{\calY}=C\). Moreover, let \(\pXY\) be the underlying (unknown) \gls*{pdf} over \(\calX\times\calY\).
Let \(\calD_{n} = \!\left\{( \x_1, y_1),\dots,( \x_n, y_n)\right\}\sim \pXY\) be a random realization of $n$ i.i.d. samples according to \(\pXY\)  denoting the \textit{training set}, where \(\x_i \in \calX\) is the input (feature), \(y_i \in \calY\) is the output class among $C$ possible classes and $n$ denotes the size of the training set.

A predictor is a function \({f}_{\calD_{n}}: \calX\rightarrow\calY\), defined as \(
f_{\calD_{n}}(\x) \triangleq \argmax_{y\in\calY}~\PYXhat(y|\x;\calD_{n}),
\), which uses the inferred model \(\PYXhat\equiv \PYXhat(y|\x;\calD_{n})\) to approximate the optimal (Bayes) decision rule
\(
f^\star(\x)  \triangleq   \argmax_{y\in\calY}~\PYX(y|\x).
\)
Notice that $\PYXhat$ can be interpreted as the prediction of the class (label) posterior probability given a sample (e.g., \(\PYXhat(y~|\x)\equiv \fun{softmax}(\mathbf{\x})_y\)), while \(\PYX\) is the true (unknown)  probability.

In general, during training, the empirical cross-entropy loss is minimized to learn the model \(\PYXhat\) from the training set \(\calD_n\) by optimizing the parameters of a model. The cross-entropy loss \(\mathcal{L}_{\mathrm{CE}}(\calD_n)\) is defined as
\begin{align}
    \label{eq:cross_entropy}
     & \mathcal{L}_{\mathrm{CE}}(\calD_n)
    \nonumber
    = -\frac{1}{n}\sum_{i=1}^{n}\Bigg[\\
                                    &\sum_{j=1}^{C}\PYX(y_j \mid \x_i)\log\PYXhat(y_j \mid \x_i;\calD_n)\Bigg].
\end{align}
Minimizing \Cref{eq:cross_entropy} encourages the model to assign high probabilities to the correct classes (cf. \cite{shalev2014understanding,bishop2006pattern,goodfellow2016deep}).
For each sample \(\x\), \Cref{eq:cross_entropy_internal} indicates the likelihood of the predicted class probabilities, i.e.,
\begin{align}
    \label{eq:cross_entropy_internal}
    \W=\sum_{j=1}^{C} p_{j}\log\left(\phat_{j}\right),
\end{align}
where \(\phat=\PYXhat(\cdot|\x;\calD_n)\) represents the predicted class probabilities for the input \(\x\) and \(p\) is the true posterior probability\footnote{In the empirical setting, the true posterior is usually represented as one-hot encoded vectors, and \Cref{eq:cross_entropy} becomes \(\mathcal{L}_{\mathrm{CE}}(\calD_n)=-\frac{1}{n}\sum_{i=1}^{n}\log\PYXhat(y_i \mid \x_i;\calD_n)\). In this case, the term in \Cref{eq:cross_entropy_internal} becomes \(\W=\log\left(\phat_{j}\right)\), where \(j\) is the index of the true class label for the input \(\x\).}.

Inspired by \cite{chow1970optimum,markowitz1952portfolio,liu2019deep}, we explicitly account for uncertainty by modifying the training objective in \Cref{eq:cross_entropy}, to enable the model to learn when to abstain. Specifically, we introduce a reservation variable represented as an additional class label, allowing the model to defer predictions when uncertainty is high. In particular, the term in \Cref{eq:cross_entropy_internal} is replaced by
\begin{align}
    \label{eq:reject_cross_entropy_internal}
    \W{r}=\sum_{j=1}^{C} p_{j}\log\left(\costmult\cdot\phat_{j}+\phat_{C+1}\right),
\end{align}
where \(\phat_{C+1}\) is the reservation variable for a given input \(\x\), subject to the constraints
\begin{align}
    \label{eq:reject_cross_entropy_sum}
    \sum_{j=1}^{C} \phat_{j}+\phat_{C+1}=1,
\end{align}
and
\begin{align}
    \label{eq:reject_cross_entropy_nonneg}
    \phat_j\ge 0,\ \forall j\in\{1,\ldots,C+1\},
\end{align}
where \(\costmult\in\R_{+}\) is a positive hyperparameter that controls the trade-off between making a prediction and abstaining. Higher values of \(\costmult\) encourages the model to make predictions, while lower values promote abstention.
Though neural networks provide a practical means of obtaining (generally suboptimal) solutions that can be applied to arbitrary inputs, as noted in \cite{Hornik1991} and \cite{liu2019deep}, a principled reformulation and analysis of the underlying optimization problem (cf. \cite{boyd2004convex}) can yield further theoretical insights. In particular, such an analysis allows us to characterize the uncertainty threshold for abstention and identify previously unrecognized degenerate cases. These findings provide further insight into the principled management of epistemic uncertainty in high-stakes applications, such as fault detection in medical devices.

\subsection{Illustrative example}
\label{subsec:illustrative_example}
We first illustrate the effect of the reservation variable on a toy classification problem with two-dimensional features and (C=3) classes in \Cref{fig:illustrative_example}. Each class has the same amount of samples. A simple fully connected neural network with two hidden layers of 32 units and ReLU activations is used. The output layer contains four units corresponding to the three target classes and one reservation output. The model is trained for 150 epochs using full-batch Adam with a learning rate of \(5\times10^{-3}\) and reservation reward \(r=2\). At inference time, an observation is accepted when its largest class probability exceeds its reservation probability and is otherwise rejected. All input features are standardized using statistics computed from the training set.

As shown in \Cref{fig:decision_regions}, introducing the reservation variable allows the model to identify and reject observations concentrated around the decision boundaries. These regions correspond to inputs for which the classifier has difficulty distinguishing between competing classes and are therefore associated with greater predictive uncertainty.

This behavior translates directly into improved performance on the retained predictions. The selective risk curve in \Cref{fig:risk_curve} shows that reducing the coverage to approximately \(81\%\) results in a reduction in classification risk of approximately \(50\%\). Thus, by abstaining on a relatively small fraction of uncertain observations, the model can substantially reduce the error rate among the predictions it retains.

Finally, \Cref{fig:margin_histogram} provides further insight into the mechanism underlying this behavior. Rejected observations are predominantly associated with smaller classification margins, indicating weaker separation between the most likely competing classes. Together, these results illustrate how the reservation mechanism enables the model to identify uncertain predictions and trade coverage for lower classification risk.

\subsection{Optimization problem with the reservation variable}
\label{subsec:optimization_problem}

We now present the main theoretical results of this work, providing a principled characterization of the role of the reservation variable in managing epistemic uncertainty. For the sake of generality, in this section, we consider true posterior probabilities \(p=\PYX(\cdot\mid\x)\), rather than empirical one-hot encoded labels. Results for one-hot encoded labels can be obtained as a special case, often reducing general results to trivial ones. Proofs based on the \gls*{kkt} optimality conditions, together with the necessary auxiliary results, are provided as supplementary material in the appendix.

\begin{restatable}{theorem}{theoremKKT}
    \label{theorem:kkt}
    Let $\hat{\mathbf p}^\star$ be a \gls*{kkt} optimal point for the maximization of \Cref{eq:reject_cross_entropy_internal} with constraints \Cref{eq:reject_cross_entropy_sum,eq:reject_cross_entropy_nonneg}, and suppose that $\hat p^\star_{C+1}>0$. Then there exists a scalar $\lambda>0$ such that, for every class $i\in\{1,\ldots,C\}$,
    \begin{align}
        \label{eq:kkt_conditions_solution}
        \hat p_i^\star = \max\left\{\frac{p_i}{\lambda} - \frac{\hat p_{C+1}^\star}{\costmult}, \,0 \right\}.
    \end{align}
\end{restatable}
Moreover, the quantity
\(
\label{eq:threshold}
\frac{\hat p_{C+1}^\star}{\costmult}
\)
defines an adaptive threshold that determines which class probabilities
remain active at the optimum. An increase in the optimal reserve probability
$\hat p_{C+1}^\star$, or a decrease in the cost multiplier $\costmult$, raises
this threshold, thereby causing a larger number of class probabilities to
become inactive. This characterization provides an optimization-based
interpretation of the reserve output as an adaptive measure of uncertainty.
The corresponding derivation is provided in \Cref{sec:proof_kkt_conditions} of the appendix.

\begin{restatable}{theorem}{theoremThreshold}
    \label{theorem:threshold}
    Consider the maximization of \Cref{eq:reject_cross_entropy_internal} subject to \Cref{eq:reject_cross_entropy_sum,eq:reject_cross_entropy_nonneg}. The cost-multiplier mixed-regime is \begin{align*} 1<\costmult<C. \end{align*} In particular, if \(\costmult<1\), full reservation is optimal, whereas if \(\costmult>C\), every optimal solution has \(\phat_{C+1}^\star=0\) and the prediction coordinates satisfy \(\phat_j^\star=p_j\) for \(j=1,\ldots,C\). The boundary cases \(\costmult=1\) and \(\costmult=C\) are degenerate: the optimizer need not be unique, and distinct prediction--reservation allocations may attain the same optimal value.
\end{restatable}
The proof is provided in \Cref{sec:proof_threshold} of the appendix.
\Cref{theorem:threshold} provides a principled characterization of the cost multiplier \(\costmult\) and its relationship to the number of classes \(C\), providing guidance for selecting \(\costmult\) in practice, as shown in the ablation studies in \Cref{sec:experiments}.

%% file: tex/tex_files/sections/experiments.tex
\input{tex/tex_files/plots/synthetic_dist_fig}

In this section, we present our numerical results and summarize the main findings through a series of plots.

We begin by examining how epistemic uncertainty affects the baseline benchmark introduced in \cite{shenoy2026uncertainty} and evaluate the extent to which our proposed approach mitigates its effects. Despite the relative simplicity of this task, it illustrates how models with otherwise strong predictive performance can remain susceptible to epistemic uncertainty, and how abstention can reduce the resulting decision risk. Furthermore, we introduce a more challenging variant of the original benchmark, in which the final detection task is more difficult by tweaking the data generative process, allowing us to evaluate the effectiveness of our mitigation approach under more realistic conditions.

Finally, we present a real-world medical device case study and introduce a benchmark that is publicly accessible through the \gls*{protookn}, demonstrating the practical applicability of our approach in a high-stakes setting.

Following standard experimental practice, we report results over five independent runs with different random seeds to account for variability due to random initialization.
The results are reported in terms of \textbf{accuracy} or its converse, i.e., empirical \textbf{risk}. As we must account for the fact that the model may abstain on some samples, we also report the \textbf{coverage}, i.e., the fraction of samples on which the model makes a prediction. Importantly, note that in case of abstention, the accuracy and risk are only computed on the samples for which the model makes a prediction, i.e., the coverage.

Thus, while our framework may abstain on some samples which may be correctly classified, \emph{increased accuracy and reduced risk} on the retained samples is a clear indication that \emph{the abstained set must contain a higher proportion of misclassified samples than the original dataset}, i.e. the abstention mechanism is preferentially rejecting predictions that are more likely to be incorrect.

For the synthetic experiments, we consider Head 1 of the multi-head cascade \gls*{mlp} architecture introduced in \cite{shenoy2026uncertainty}, which performs defect classification. The model consists of two hidden layers with 256 neurons each, ReLU activations, and a dropout rate of \(0.1\), and it is trained using Adam with a learning rate of \(10^{-3}\) and a batch size of 128 for up to 40 epochs, with early stopping using a patience of 6 epochs.

For the real-world experiments, the model consists of two hidden layers with 256 and 128 neurons, respectively, and ReLU activations, and it is trained using Adam \cite{kingma2014method} with a learning rate of \(10^{-3}\), weight decay of \(10^{-4}\), and a batch size of 32 for up to 80 epochs, with early stopping using a patience of 10 epochs.

All models are implemented in PyTorch \cite{paszke2019pytorch}. Development and preliminary experiments were conducted on Apple Silicon, while the main experimental sweeps were executed on a SLURM-based computing cluster using NVIDIA A100 and H100 GPUs, with one GPU allocated per run.

The FDA MAUDE and 510(k) source data are publicly available from the U.S.~Food and Drug Administration~\cite{fda_maude,fda_510k}. The integrated Medical Device Knowledge Graph currently comprises 17.7 million triples across 29 classes and 887{,}200 subjects; it is registered in the Open Knowledge Network registry (\url{https://registry.okn.us/registry/kgs/medical-device-kg/}) and publicly queryable through the FRINK SPARQL endpoint (\url{https://apps.okn.us/medical-device-kg/sparql}).
\subsection{Synthetic Benchmark}
\label{subsec:synthetic-benchmark}
We first consider the benchmark introduced in \cite{shenoy2026uncertainty}, which we refer to as the \emph{baseline}. To the best of our knowledge, the original benchmark data, which is also synthetic, has not been publicly released. We therefore reproduce the data-generation process following the description provided in the original work (cf.~\Cref{subsec:benchmark_gen,alg:generator}). To assess the fidelity of our reconstruction, we verify that a comparable \gls*{mlp} achieves performance consistent with that reported in the original work. Given the synthetic nature and relative simplicity of this benchmark, we primarily use it as a proof of concept to illustrate the effects of epistemic uncertainty and the benefits of abstention under controlled conditions. We subsequently introduce a more challenging synthetic scenario to evaluate whether these benefits persist as the underlying detection task becomes more complex.
\input{tex/tex_files/plots/synthetic_cascade_vs_abstention_fig}
The problem at hand considers a three-class classification task (\emph{no defect}, \emph{open circuit}, and \emph{solder bridging}) based on six process features: ambient relative humidity, ambient temperature, paste viscosity, stencil thickness, time above liquidus, and peak reflow temperature.
The per-feature class distributions are shown in \Cref{fig:baseline}. As can be observed, several features provide a relatively clear separation between classes, making the resulting classification task comparatively simple. Nevertheless, \Cref{fig:cascade-vs-abstention-synthetic} shows that when abstaining on less than \(10\%\) of the samples, the accuracy increases from \(95.10\%\) to \(98.17\%\), an improvement of \(3.07\) percentage points, corresponding to a \(62.65\%\) relative reduction in classification error. This indicates that the abstention mechanism preferentially identifies predictions that are more likely to be misclassified, thereby improving the accuracy of the retained decisions.

Next, we artificially increase the difficulty of the classification task by increasing the overlap between the class-conditional feature distributions, thereby increasing the Bayes error, as illustrated in \Cref{fig:overlap_harder}. As expected, the overall classification accuracy decreases under this more challenging setting. Nevertheless, \Cref{fig:cascade-vs-abstention-synthetic} shows that abstention continues to improve the accuracy of the retained predictions, demonstrating the effectiveness of the approach even as the underlying classification problem becomes more difficult. In this case, abstaining yields an accuracy increase from \(87.59\%\) to \(89.94\%\), an improvement of \(2.35\) percentage points, corresponding to an \(18.94\%\) relative reduction in classification error.

Finally, we present an ablation study across different target coverage levels in \Cref{fig:risk-baseline} for the baseline setting and \Cref{fig:risk-harder} for the higher-overlap setting. An additional ablation study across different payoff values, which further supports the results reported in \Cref{fig:selective-risk}, is provided in the appendix (cf. \Cref{subsec:additional_numerical_results,fig:ablation-k-synthetic}).

\input{tex/tex_files/plots/synthetic_ablation_study_fig}

\subsection{Experiment on real medical devices data}
\label{subsec:real-medical-devices-data}

In this section, we present the results of our experiments on real-world medical device data. The dataset consists of report narratives, extracted from \cite{fda_maude,fda_510k} through the contributed \gls*{kg} within the \gls*{protookn} project (cf. \cite{protookn}).
\input{tex/tex_files/plots/maude_ablation_fig.tex}
The final dataset comprises 266,170 reports (185,477 malfunction, 44,995 injury, 33,897 serious injury, 1801 death), partitioned into training (64\%), validation (16\%), and test (20\%) splits and stratified by outcome class. Severity labels are derived from the MAUDE event type field as follows: reports typed Death and Malfunction map directly to the Death and Malfunction classes, while Injury reports are split into basic vs. serious injury. An injury is classified as serious if its narrative matches a serious-injury cue and basic injury otherwise\footnote{Keywords that identify serious injury are the ones that are described as life threatening situation, permanent impairment and surgical interventions.}, reports typed Other or No answer, or carrying no narrative are discarded. To prevent label leakage, sentences in the narratives whose wording contains or closely matches the outcome labels are removed. Duplicate reports were removed based on exact narrative text duplications.

These narratives consist of structured key-value pairs, including textual descriptions of adverse events associated with medical devices. The datasets are connected by the device identifiers, which are used to link the narratives to the corresponding device information.

The task we select is to classify these text-based features into four outcome classes: malfunction, basic injury, serious injury, and death, in increasing order of severity. The underlying objective is to identify patterns associated with adverse events in previously unseen data, with the broader goal of supporting their early detection and potentially preventing similar events from occurring in the future.
Text content is represented using TF-IDF features with up to 50,000 unigram and bigram features, selected for scalability, interpretability, and independence from pretrained language models. Further details are provided in the publicly available code repository.

\input{tex/tex_files/plots/maude_barplot_fig.tex}

\Cref{fig:maude-k} and \Cref{fig:maude-coverage} show the effectiveness of abstention on the retained predictions, confirming the expected results. Additionally, for this task, it is possible to achieve almost perfect accuracy for \(\approx 70\%\) coverage. In operational terms, abstention defers roughly one in ten test reports to human review in exchange for a substantially lower error rate among retained predictions, consistent with triage workflows in post-market surveillance.
This indicates that the abstention mechanism is preferentially rejecting predictions that are more likely to be incorrect, thereby improving the accuracy of the retained predictions, though nearly optimal accuracy may incur a high cost in terms of the number of predictions that are abstained from. The ablation study across different target coverage levels in \Cref{fig:maude-k} further supports this conclusion, demonstrating that the benefits of abstention persist across a range of coverage levels.

Finally, \Cref{fig:maude-noabstain-vs-abstain} shows the performance comparison between the standard cross-entropy model and its version equipped with abstention, showing consistent and favorable results across multiple independent runs.

\section{Limitations and Future Work}
\label{subsec:limitations-and-future-work}

Given the importance of uncertainty-aware decision-making in medical device manufacturing, this work represents an initial step toward systematically studying epistemic uncertainty in this domain. A key limitation is the current availability of real-world benchmarks that capture both manufacturing defects and the uncertainty associated with their detection. Consequently, broader evaluation across devices, defect types, and manufacturing settings remains necessary to establish the generalizability of these findings.

Two characteristics of the underlying data source should be kept in mind when interpreting these results. MAUDE is a passive, spontaneous-report surveillance system: report volumes reflect reporting behavior, device utilization, and regulatory reporting requirements rather than confirmed failure incidence, and narratives vary widely in completeness and may
contain duplicates~\cite{chung2020signal}. Accordingly, the benchmark task should be understood as report-level severity classification under realistic label noise, a regime in which epistemic uncertainty is expected to be substantial and abstention correspondingly valuable, and
not as an estimate of device failure rates. These characteristics make MAUDE-derived data a stringent, practically relevant testbed for uncertainty-aware classification~\cite{mao2021nlp,sohoni2026llm}. Additionally, duplicates were removed by exact narrative matching, yet MAUDE contains near-duplicate reports of the same event (e.g., manufacturer and user-facility submissions), which random splitting may place on both sides of the train--test split. Near-duplicate detection and temporal splits are natural extensions to better approximate deployment conditions.

We hope this work will encourage further research on epistemic uncertainty in medical device manufacturing, particularly through the development of larger, more diverse, and application-relevant benchmarks. Resources such as the \gls*{protookn}, the data repositories in \cite{fda_510k,fda_maude}, together with other tools for integrating and organizing heterogeneous medical device data, provide an opportunity to facilitate the construction and sharing of such datasets. Such efforts could support more comprehensive evaluations of uncertainty-aware methods and provide a stronger empirical foundation for their use in medical device quality assurance.

%% file: tex/tex_files/plots/synthetic_dist_fig.tex
\begin{figure*}[!htbp]
    \centering

    \begin{subfigure}{1.\textwidth}
        \centering
        \includegraphics[width=\textwidth]{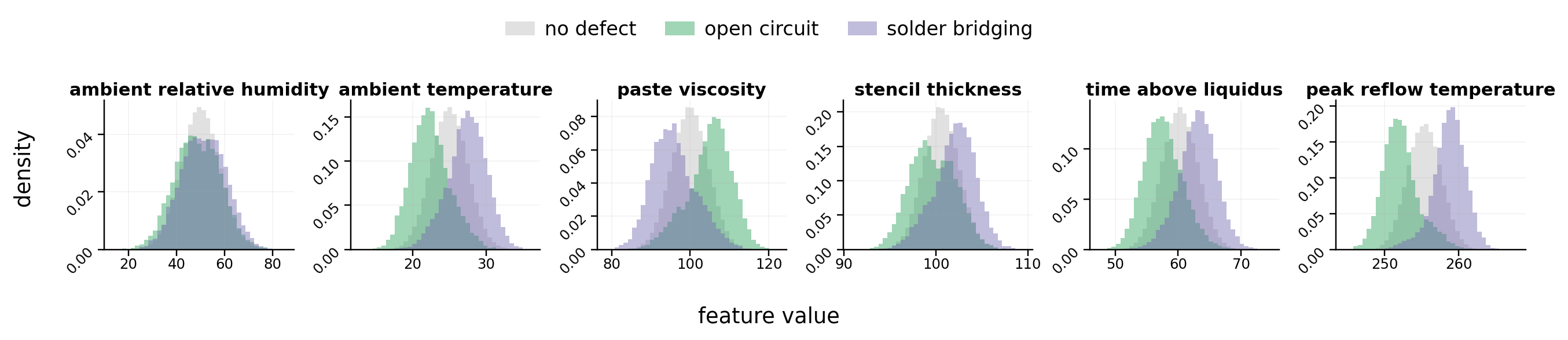}
        \caption{Baseline.}
        \label{fig:baseline}
    \end{subfigure}

    \vspace{0.5em}

    \begin{subfigure}{1.\textwidth}
        \centering
        \includegraphics[width=\textwidth]{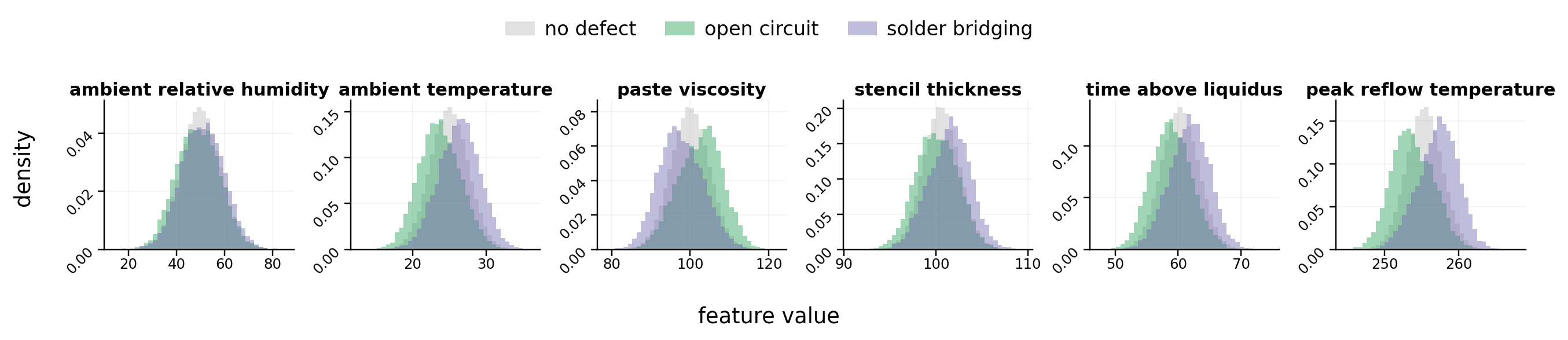}
        \caption{Increased class posterior overlap.}
        \label{fig:overlap_harder}
    \end{subfigure}

    \caption{Synthetic benchmark data -- Feature distributions under different levels of class posterior overlap. \Cref{fig:baseline} shows the setting for the benchmark introduced in \cite{shenoy2026uncertainty}, while \Cref{fig:overlap_harder} shows a more challenging setting with increased overlap between the class-conditional feature distributions, controlled by the target Bayes error parameter in \Cref{alg:generator}.}
    \label{fig:feature-separability}
\end{figure*}

%% file: tex/tex_files/plots/synthetic_cascade_vs_abstention_fig.tex
\begin{figure}[!htbp]
    \includegraphics[width=.48\textwidth]{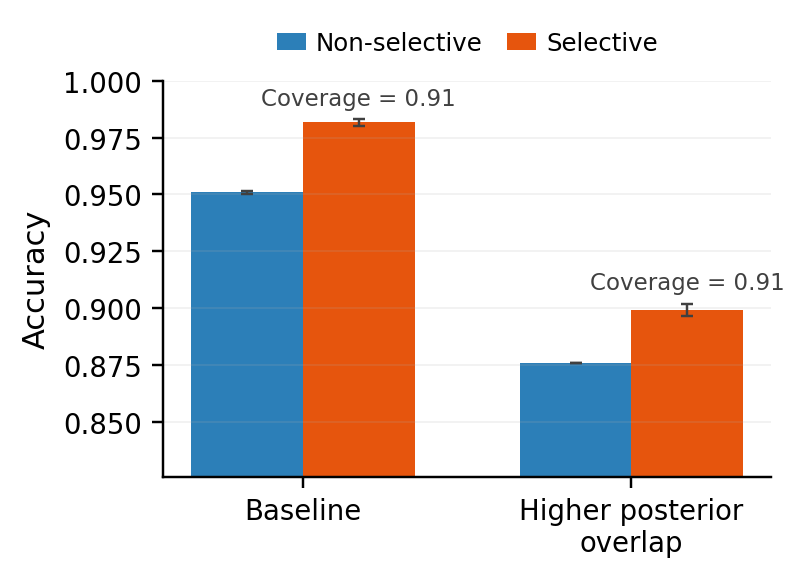}
    \caption{Synthetic benchmark data -- (Best) Accuracy bar plot over multiple seeds for the baseline benchmark, comparing the performance of a single model w/o and w/ abstention (target coverage \(\ge 90\%\)). 
    For the baseline setting, abstention increases mean accuracy from \(95.10\%\) to \(98.17\%\) (\(+3.07\) percentage points), reducing classification error by \(62.65\%\). At higher overlap, it increases mean accuracy from \(87.59\%\) to \(89.94\%\) (\(+2.35\) percentage points), reducing classification error by \(18.94\%\).
    The narrow error bars, which represent variability across multiple random seeds, indicate that performance remains consistent under different random initializations.
    }
    \label{fig:cascade-vs-abstention-synthetic}
\end{figure}

%% file: tex/tex_files/plots/synthetic_ablation_study_fig.tex
\begin{figure}[!htbp]
    \centering
    \begin{subfigure}[t]{0.48\linewidth}
        \centering
        \includegraphics[width=\linewidth]{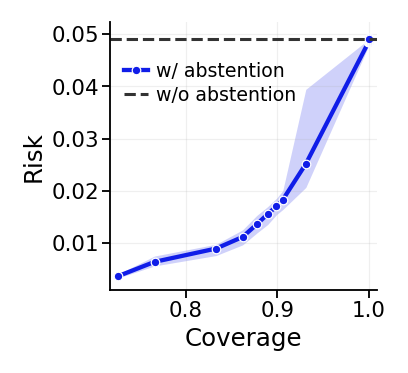}
        \caption{Baseline}
        \label{fig:risk-baseline}
    \end{subfigure}
    \hspace{0.005\linewidth}
    \begin{subfigure}[t]{0.48\linewidth}
        \centering
        \includegraphics[width=\linewidth]{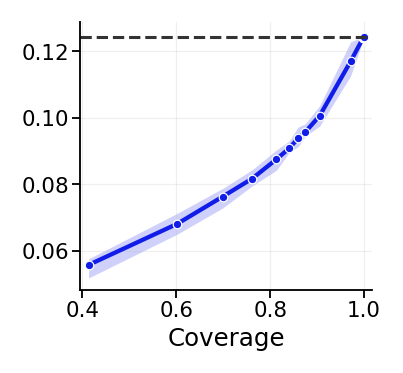}
        \caption{Increased class posterior overlap}
        \label{fig:risk-harder}
    \end{subfigure}
    \caption{Synthetic benchmark data -- Risk as a function of coverage. Results are shown for the baseline setting (\Cref{fig:risk-baseline}) and the more challenging setting with increased class posterior overlap (\Cref{fig:risk-harder}). In both cases, the risk decreases as the coverage decreases, indicating that the abstention mechanism preferentially rejects predictions that are more likely to be misclassified. Results over several random seeds are shown, with the shaded area indicating the standard deviation, showing that the results are consistent across different independent runs.}
    \label{fig:selective-risk}
\end{figure}

%% file: tex/tex_files/plots/maude_ablation_fig.tex
\begin{figure}[t]
    \centering
    \begin{subfigure}[t]{0.51\columnwidth}
        \centering
        \includegraphics[width=\linewidth]{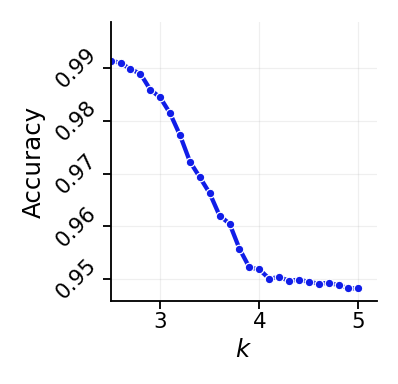}
        \caption{Accuracy vs.\ payoff $k$.}
        \label{fig:maude-k}
    \end{subfigure}
    \hspace{-.05\columnwidth}
    \begin{subfigure}[t]{0.49\columnwidth}
        \centering
        \includegraphics[width=\linewidth]{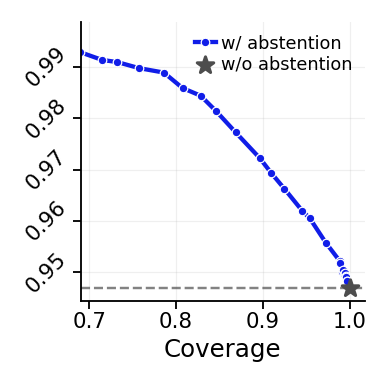}
        \caption{Accuracy vs.\ coverage.}
        \label{fig:maude-coverage}
    \end{subfigure}

    \caption{Medical device data -- Performance across different values of \(k\) and coverage.}
    \label{fig:maude-selective}
\end{figure}

%% file: tex/tex_files/plots/maude_barplot_fig.tex
\begin{figure}
    \includegraphics[width=.48\textwidth]{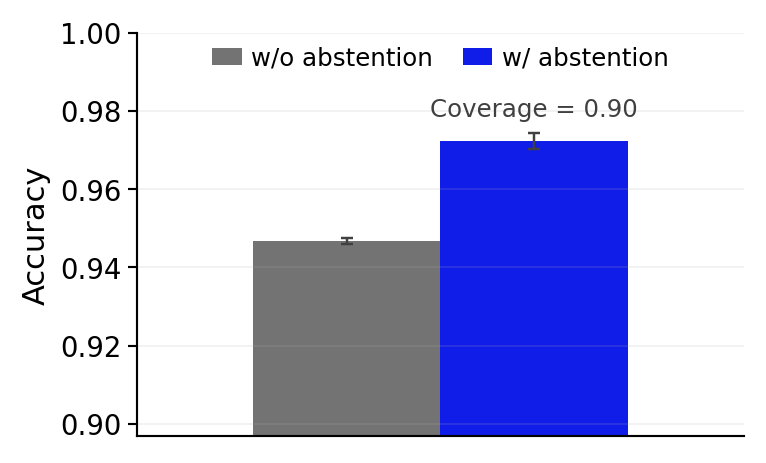}
    \caption{Medical device data -- Performance comparison between no abstention and abstention. Results are reported over several random seeds, showing narrow error bars indicating consistent results across multiple independent runs. The abstention mechanism allows for an increase in average accuracy, from \(94.68\%\) to \(97.23\%\), an improvement of \(2.55\) percentage points, corresponding to a \(47.93\%\) relative reduction in classification error with a target coverage of \(90\%\), further demonstrating its effectiveness.
    }
    \label{fig:maude-noabstain-vs-abstain}
\end{figure}

%% file: tex/tex_files/sections/conclusion.tex
In this work, we investigated the impact of epistemic uncertainty on defect detection in medical device manufacturing, a first step particularly important when incorrect predictions can propagate through subsequent root-cause analysis and risk assessment. We showed that explicitly accounting for this uncertainty through abstention provides a simple yet principled mechanism for reducing the risk associated with unreliable model decisions.

Building on a cost-sensitive classification framework, we analyzed a rejection strategy that allows the model to abstain from uncertain predictions rather than forcing a decision for every input. Our experiments on synthetic benchmarks show that abstaining on a relatively small fraction of predictions can reduce decision risk while improving the reliability of the remaining decisions. We further evaluated the approach using real-world medical device report data, demonstrating its applicability beyond controlled synthetic settings. To support reproducibility and further research, we have shared a link to our open-source code.

More broadly, our results suggest that predictive performance alone may not adequately characterize model reliability in high-stakes defect detection. Providing models with the ability to identify and defer uncertain decisions offers a practical mechanism for incorporating uncertainty into automated quality assurance and provides a foundation for integrating such decisions with downstream root-cause analysis, risk assessment, and human review.

%% file: tex/tex_files/sections/appendix.tex
In this section we present the proofs for the main theoretical results in \Cref{sec:framework}. We first present some auxiliary results that will be used in the proofs of \Cref{theorem:kkt,theorem:threshold}. For the sake of generality, in this section, we consider true posterior probabilities \(p=\PYX(\cdot\mid\x)\), rather than empirical one-hot encoded labels. Indeed, let \(Z=(Z_1,\ldots,Z_C)\) denote the observed one-hot target, where
\begin{align}
    Z_i = \mathbbm{1}\{Y=i\},
    \qquad i\in\{1,\ldots,C\}.
\end{align}
Thus, for every individual observation, \(Z_i\in\{0,1\}\) and
\begin{align}
    \sum_{i=1}^{C} Z_i = 1.
\end{align}
We distinguish this observed target from the true conditional class
probabilities
\begin{align}
    p_i(\x)
    := \mathbb{P}(Y=i\mid X=\x)
    = \mathbb{E}[Z_i\mid X=\x],
\end{align}
for \(i\in\{1,\dots,C\}\).
Hence, although the observed target \(Z\) is one-hot, the conditional
posterior
\begin{align}
    \mathbf{p}(\x)
    =
    \bigl(p_1(\x),\ldots,p_C(\x)\bigr)
\end{align}
need not be one-hot and satisfies
\begin{align}
    p_i(\x)\geq 0,
    \qquad
    \sum_{i=1}^{C}p_i(\x)=1.
\end{align}

For a single observation with one-hot target \(Z\), consider the
objective
\begin{align}
    \sum_{i=1}^{C}
    Z_i
    \log\left(
    \costmult\phat_i+\phat_{C+1}
    \right).
\end{align}
Taking the conditional expectation with respect to the label \(Y\)
given \(X=\x\) yields
\begin{align}
     & \mathbb{E}\left[
                     \left.
                     \sum_{i=1}^{C}
                     Z_i
                     \log\left(
                     \costmult\phat_i+\phat_{C+1}
                     \right)
                     \,\right|\,X=\x
                     \right]
    \nonumber                                    \\
     & \qquad=
    \sum_{i=1}^{C}
    \mathbb{E}\left[
                  Z_i\mid X=\x
                  \right]
    \log\left(
    \costmult\phat_i+\phat_{C+1}
    \right)
    \nonumber                                    \\
     & \qquad=
    \sum_{i=1}^{C}
    p_i(\x)
    \log\left(
    \costmult\phat_i+\phat_{C+1}
    \right).
\end{align}
Therefore, the quantities \(p_i(\x)\) appearing in the theoretical
optimization should be interpreted as the true conditional class
probabilities rather than as the observed one-hot labels. The empirical
loss is evaluated using one-hot targets, whereas the corresponding
conditional population objective is obtained by averaging this loss
over the conditional distribution of \(Y\) given \(X=\x\).

\subsection{Auxiliary results}
\label{subsec:auxiliary_results}
First, let us consider the problem of optimizing \Cref{eq:reject_cross_entropy_internal} in the probability simplex
\begin{align}
    \label{eq:probability_simplex}
    \nonumber
    \Delta^{C} = \Bigg\{\placeholderdist \in \mathbb{R}^{C+1}: & \; \placeholderdist_i \ge 0 \; \forall\ i\in\{1, \dots, C+1\}, \\
                                                               & \sum_{i=1}^{C+1} \placeholderdist_i = 1 \Bigg\}
\end{align}
\begin{restatable}{lemma}{lemmaSimplex}
    \label{lemma:simplex_properties}
    The probability simplex \(\Delta^{C}\) defined in \Cref{eq:probability_simplex} is convex, closed, and bounded.
\end{restatable}
\begin{proof}
    Let $\placeholderdist,\placeholderdist^{\prime} \in \Delta^C$ and let $\gamma\in[0,1]$. Define
    \begin{align*}
        \placeholderdist^{\prime\prime} = \gamma \placeholderdist + (1-\gamma)\placeholderdist^{\prime}.
    \end{align*}
    For every coordinate $i$,
    \begin{align*}
        \placeholderdist^{\prime\prime}_i = \gamma \placeholderdist_i + (1-\gamma)\placeholderdist^{\prime}_i \ge 0,
    \end{align*}
    since $\placeholderdist_i,\placeholderdist^{\prime}_i\ge 0$ and $\gamma,1-\gamma\ge 0$. Moreover,
    \begin{align*}
        \sum_{i=1}^{C+1} \placeholderdist^{\prime\prime}_i = & \sum_{i=1}^{C+1}\bigl(\gamma \placeholderdist_i+(1-\gamma)\placeholderdist^{\prime}_i\bigr)        \\
        =                                                    & \gamma\sum_{i=1}^{C+1} \placeholderdist_i + (1-\gamma)\sum_{i=1}^{C+1} \placeholderdist^{\prime}_i \\
        =                                                    & \gamma+(1-\gamma) =  1.
    \end{align*}
    Hence $\placeholderdist^{\prime\prime}\in\Delta^C$, proving that $\Delta^C$ is convex. Next, observe that
    \begin{align*}
        \Delta^C = & \left(\bigcap_{i=1}^{C+1}\{\placeholderdist\in\mathbb R^{C+1}: \placeholderdist_i\ge 0\}\right) \\
                   & \cap
        \left\{\placeholderdist\in\mathbb R^{C+1}:\sum_{i=1}^{C+1} \placeholderdist_i = 1\right\}.
    \end{align*}
    Each set
    \begin{align*}
        \{\placeholderdist\in\mathbb R^{C+1}: \placeholderdist_i\ge 0\}
    \end{align*}
    is closed, as it is the preimage of the closed interval
    $[0,\infty)$ under the continuous projection map
    $\placeholderdist\mapsto \placeholderdist_i$. Likewise,
    \begin{align*}
        \left\{
        \placeholderdist\in\mathbb R^{C+1}:
        \sum_{i=1}^{C+1} \placeholderdist_i = 1
        \right\}
    \end{align*}
    is closed, being the preimage of the closed set $\{1\}$
    under the continuous map
    $\placeholderdist\mapsto \sum_{i=1}^{C+1} \placeholderdist_i$.
    Since an intersection of closed sets is closed,
    $\Delta^C$ is closed. Finally, let $\placeholderdist\in\Delta^C$. Since $\placeholderdist_i\ge 0$ and
    \begin{align*}
        \sum_{i=1}^{C+1} \placeholderdist_i = 1,
    \end{align*}
    it follows that
    \(
    0 \le \placeholderdist_i \le 1,
    \;\text{for}\;
    i=1,\dots,C+1.
    \)
    Therefore,
    \begin{align*}
        \|\placeholderdist\|_2^2
        =
        \sum_{i=1}^{C+1} \placeholderdist_i^2
        \le
        \sum_{i=1}^{C+1} \placeholderdist_i
        =
        1,
    \end{align*}
    where the inequality follows from $0\le \placeholderdist_i\le 1$ implying
    $\placeholderdist_i^2\le \placeholderdist_i$. Consequently,
    \(
    \|\placeholderdist\|_2 \le 1,
    \)
    and thus $\Delta^C$ is bounded.
\end{proof}
Moreover, the following holds:
\begin{restatable}{lemma}{lemmaConcavity}
    \label{lemma:concavity}
    The function in \Cref{eq:reject_cross_entropy_internal} is concave in \(\PYXhat\) over the domain \(\Delta^{C}\).
\end{restatable}
\begin{proof}
    To show that $\W{r}$ is concave, it suffices to show that $-\W{r}$ is convex. Observe that
    \begin{align*}
        -\W{r} = \sum_{j=1}^{C}p_j\left[-\log\!\left(\costmult\cdot\phat_j+\phat_{C+1}\right)\right].
    \end{align*}
    For each $j$, define
    \(
    g_j(\phat) = -\log\!\left(\costmult\cdot\phat_j+\phat_{C+1}\right).
    \)
    Since $\costmult>0$ and $\phat\in\Delta^C$, we have
    \(
    \costmult\cdot\phat_j+\phat_{C+1}\geq 0.
    \)
    The function $x\mapsto-\log x$ is convex\footnote{We adopt the extended-value convention
    \(
    -\log 0 := +\infty.
    \)
    Thus, under the above convention, $x\mapsto-\log x$ defines an
    extended-real-valued convex function on $[0,\infty)$. Moreover, the map
    \(
    \phat\mapsto\costmult\cdot\phat_j+\phat_{C+1}
    \)
    is affine. Hence, by preservation of convexity under affine
    precomposition, each
    \(
    g_j(\phat) = -\log\!\left(\costmult\cdot\phat_j+\phat_{C+1}\right)
    \)
    is convex on $\Delta^C$ in the extended-real sense.
    } on $(0,\infty)$, since
    \begin{align*}
        \frac{d^2}{dx^2}(-\log x)
        =
        \frac{1}{x^2}
        >
        0,
        \qquad x>0,
    \end{align*}
    and satisfies
    \(
    \lim_{x\rightarrow 0} -\log x = +\infty.
    \)
    Since $p_j\geq 0$
    for all $j$, the nonnegative weighted sum
    \begin{align*}
        -\W{r}
        =
        \sum_{j=1}^{C} p_j g_j
    \end{align*}
    is convex. Therefore, $\W{r}$ is concave on $\Delta^C$ in the
    extended-real sense, with the corresponding convention $\log 0=-\infty$.
\end{proof}
From \Cref{lemma:simplex_properties,lemma:concavity}, the following can be concluded:
\begin{restatable}{remark}{remarkadmissible}
    \label{remark:admissible}
    The optimization problem of maximizing \(\W{r}\) over the probability simplex \(\Delta^{C}\) involves the maximization over a convex, closed, and bounded set of a concave function. Hence, the problem is a \gls*{cmp} problem, and for the continuity of \(\W{r}\) over the compact set \(\Delta^{C}\), it admits a global maximum.
\end{restatable}
Due to the nature of the problem analyzed, we can build the following \gls*{cmp} problem, which is equivalent to the maximization of \(\W{r}\) over the probability simplex \(\Delta^{C}\):
\begin{align}
    \label{eq:cmp_problem}
    \left\{
    \begin{aligned}
        \text{maximize}\quad   & \W{r}                              \\
        \text{subject to}\quad & \sum_{i=1}^{C+1} \phat_i = 1,      \\
                               & \phat_i \ge 0, \quad i=1,\dots,C+1
    \end{aligned}
    \right.
\end{align}
Thus, the following holds:
\begin{restatable}{proposition}{propEquivalence}
    \label{prop:equivalence}
    The formulation in \Cref{eq:cmp_problem} encompasses the standard training with \Cref{eq:cross_entropy} as a special case when \(\phat_{C+1}=0\) (cf. \cite{liu2019deep}), as it reduces to the minimization of the Kullback-Leibler divergence between the true and predicted distributions, which is equivalent to the cross-entropy loss.
\end{restatable}
\begin{proof}

    Let $\phat\in\Delta^C$ with $\phat_{C+1}=0$. Then
    \begin{align*}
        \W{r}
         & =
        \sum_{j=1}^{C}p_j\log\!\left(\costmult\cdot\phat_j+\phat_{C+1}\right)             \\
         & =
        \sum_{j=1}^{C}p_j\log\!\left(\costmult\cdot\phat_j\cdot\frac{p_j}{\phat_j}\right) \\
         & =
        \log\!\left(\costmult\right)
        -
        H(p)
        -
        D_{\text{KL}}(p\| \phat),
    \end{align*}
    where $H(p)$ is the Shannon entropy of $p$ and $D_{\text{KL}}(p\| \phat)$ is the Kullback-Leibler divergence between $p$ and $\phat$. Since $\log\!\left(\costmult\right)$ and $H(p)$ are constants with respect to $\phat$, maximizing $\W{r}$ is equivalent to minimizing $D_{\text{KL}}(p\| \phat)$, which is equivalent to minimizing the cross-entropy loss.
\end{proof}
\subsection{Proof of \Cref{theorem:kkt}}
\label{sec:proof_kkt_conditions}
\theoremKKT*
\begin{proof}
    The optimization problem in Eq.~\eqref{eq:cmp_problem} can be analyzed using the \gls*{kkt} conditions. The corresponding Lagrangian is given by
    \begin{align*}
        \mathcal{L}(\phat, \lambda, \mu) & = \W{r} \\ & +\lambda\left(1-\sum_{i=1}^{C+1}\phat_i\right) \\ & +\sum_{i=1}^{C+1}\mu_i\phat_i,
    \end{align*}
    where $\lambda\in\mathbb{R}$ is the Lagrange multiplier associated with the equality constraint and $\mu_i\ge0$ are the multipliers corresponding to the non-negativity constraints.
    The \gls*{kkt} conditions are
    \begin{align*}
        \left\{
        \begin{aligned}
             & \frac{\partial \mathcal{L}}{\partial \phat_i}
            = \frac{\partial \W{r}}{\partial \phat_i}
            - \lambda + \mu_i = 0,
             &                                               & i=1,\ldots,C+1, \\
             & \sum_{i=1}^{C+1} \phat_i = 1,                                   \\
             & \phat_i                   \ge 0,
             &                                               & i=1,\ldots,C+1, \\
             & \mu_i                       \ge 0,
             &                                               & i=1,\ldots,C+1, \\
             & \mu_i \phat_i             = 0,
             &                                               & i=1,\ldots,C+1.
        \end{aligned}
        \right.
    \end{align*}
    For the objective in \Cref{eq:reject_cross_entropy_internal}, the stationarity conditions yield
    \begin{align*}
        \frac{\partial\mathcal{L}}{\partial\phat_i} & = \frac{\costmult\cdot p_i}{\costmult\cdot\phat_i+\phat_{C+1}}-\lambda+\mu_i = 0,
    \end{align*}
    for \(i = 1, \ldots, C\), and
    \begin{align*}
        \frac{\partial\mathcal{L}}{\partial\phat_{C+1}} & = \sum_{i=1}^{C}\frac{p_i}{\costmult\cdot\phat_i+\phat_{C+1}}-\lambda+\mu_{C+1} = 0.
    \end{align*}
    Since $\phat_{C+1}^\star>0$, complementary slackness implies $\mu_{C+1}=0$. Evaluating the stationarity condition at $\hat{\mathbf p}^\star$ therefore gives
    \begin{align*}
        \lambda = \sum_{i=1}^{C}\frac{p_i}{\costmult\cdot\phat_i^\star+\phat_{C+1}^\star}.
    \end{align*}
    Since $\phat_{C+1}^\star>0$, every denominator is strictly positive. Moreover, since $\mathbf p$ is a probability vector, $p_i\ge0$ for all $i$ and at least one $p_i>0$. Hence,
    \begin{align*}
        \lambda>0.
    \end{align*}

    The complementary slackness condition naturally leads to two cases for each class output.

    \paragraph{Case 1: $\phat_i^\star>0$}
    In this case, complementary slackness implies $\mu_i=0$, yielding
    \begin{align*}
        \frac{\costmult\cdot p_i}{\costmult\cdot\phat_i^\star+\phat_{C+1}^\star} = \lambda,
    \end{align*}
    or equivalently,
    \begin{align*}
        \phat_i^\star = \frac{p_i}{\lambda}-\frac{\phat_{C+1}^\star}{\costmult}.
    \end{align*}
    Notice how the requirement $\phat_i^\star>0$ implies that
    \begin{align*}
        p_i>\lambda\frac{\phat_{C+1}^\star}{\costmult},
    \end{align*}
    which means that a class becomes inactive whenever its probability falls below the threshold \(\lambda\frac{\phat_{C+1}^\star}{\costmult}\), where \(\frac{\phat_{C+1}^\star}{\costmult}\) is the ratio of the reserve output to the cost multiplier associated with betting on the class prediction.

    \paragraph{Case 2: $\phat_i^\star=0$}
    In this case, complementary slackness only implies that $\mu_i\ge0$. The stationarity condition becomes
    \begin{align*}
        \frac{\costmult\,p_i}{\phat_{C+1}^\star}-\lambda+\mu_i=0,
    \end{align*}
    where $\phat_{C+1}^\star>0$ by assumption. Equivalently,
    \begin{align*}
        \mu_i = \lambda-\frac{\costmult\,p_i}{\phat_{C+1}^\star}.
    \end{align*}
    Since $\mu_i\ge0$, it follows that
    \begin{align*}
        \frac{\costmult\,p_i}{\phat_{C+1}^\star}\le\lambda,
    \end{align*}
    or equivalently,
    \begin{align*}
        p_i\le\lambda\frac{\phat_{C+1}^\star}{\costmult}.
    \end{align*}
    Since $\lambda>0$, this is equivalent to
    \begin{align*}
        \frac{p_i}{\lambda}-\frac{\phat_{C+1}^\star}{\costmult}\le0.
    \end{align*}

    Combining the two cases, for every $i\in\{1,\ldots,C\}$ we obtain
    \begin{align*}
        \phat_i^\star = \max\left\{\frac{p_i}{\lambda}-\frac{\phat_{C+1}^\star}{\costmult}, \,0\right\},
    \end{align*}
    which proves the result.
\end{proof}

\subsection{Proof of \Cref{theorem:threshold}}
\label{sec:proof_threshold}
\theoremThreshold*
\begin{proof}
    We consider separately the regimes \(\costmult<1\), \(\costmult>C\), and the boundary cases \(\costmult=1\) and \(\costmult=C\).

    \paragraph{Case 1: \(\costmult<1\).}
    For any feasible \(\phat\) and every \(j\in\{1,\ldots,C\}\),
    \begin{align*}
        \costmult\phat_j+\phat_{C+1}\le\phat_j+\phat_{C+1}\le1,
    \end{align*}
    where the second inequality follows from \(\sum_{j=1}^{C+1}\phat_j=1\). Hence,
    \begin{align*}
        \log\left(\costmult\phat_j+\phat_{C+1}\right)\le0,
    \end{align*}
    and therefore
    \begin{align*}
        \W{r}\le0.
    \end{align*}
    The full-reservation solution
    \begin{align*}
        (\phat_1,\ldots,\phat_C,\phat_{C+1})=(0,\ldots,0,1)
    \end{align*}
    attains \(\W{r}=0\), and is therefore optimal. Thus, for \(\costmult<1\), full reservation is optimal.

    \paragraph{Case 2: \(\costmult>C\).}
    Suppose, for contradiction, that an optimal solution satisfies \(\phat_{C+1}^\star>0\). Let
    \begin{align*}
        S=\{j:\phat_j^\star>0\}, \qquad a=|S|, \qquad P_{\bar S}=\sum_{j\notin S}p_j.
    \end{align*}
    Since \(\phat_{C+1}^\star>0\), complementary slackness gives \(\mu_{C+1}=0\). Moreover, for every \(j\in S\), \Cref{eq:kkt_conditions_solution} implies
    \begin{align*}
        \costmult\phat_j^\star+\phat_{C+1}^\star=\frac{\costmult p_j}{\lambda}.
    \end{align*}
    Hence, the stationarity condition for the reservation coordinate becomes
    \begin{align*}
        \frac{a\lambda}{\costmult}+\frac{P_{\bar S}}{\phat_{C+1}^\star}=\lambda.
    \end{align*}

    If \(P_{\bar S}>0\), then
    \begin{align*}
        P_{\bar S}=\lambda\left(1-\frac{a}{\costmult}\right)\phat_{C+1}^\star.
    \end{align*}
    Summing \Cref{eq:kkt_conditions_solution} over \(j\in S\) and using \(\sum_{j\in S}\phat_j^\star=1-\phat_{C+1}^\star\) gives
    \begin{align*} \frac{1-P_{\bar S}}{\lambda}-\frac{a\phat_{C+1}^\star}{\costmult}=1-\phat_{C+1}^\star. \end{align*}
    Substituting the preceding expression for \(P_{\bar S}\) yields
    \(\lambda=1 \).
    Therefore,
    \begin{align*}
        P_{\bar S}=(\costmult-a)\frac{\phat_{C+1}^\star}{\costmult}.
    \end{align*}
    On the other hand, for every inactive coordinate \(j\notin S\), \Cref{eq:kkt_conditions_solution} implies
    \begin{align*}
        p_j\le\frac{\phat_{C+1}^\star}{\costmult}.
    \end{align*}
    Summing over the \(C-a\) inactive coordinates gives
    \begin{align*}
        P_{\bar S}\le(C-a)\frac{\phat_{C+1}^\star}{\costmult}.
    \end{align*}
    Combining the last two relations and using \(\phat_{C+1}^\star>0\) yields
    \begin{align*}
        \costmult-a\le C-a,
    \end{align*}
    and hence
    \begin{align*}
        \costmult\le C,
    \end{align*}
    contradicting the assumption \(\costmult>C\).

    It remains to consider \(P_{\bar S}=0\). In this case, the reservation stationarity condition reduces to
    \begin{align*} \frac{a\lambda}{\costmult}=\lambda. \end{align*}
    Since \(\lambda>0\), this implies
    \begin{align*} a=\costmult. \end{align*}
    However, \(a\le C<\costmult\), which is impossible. Consequently,
    \begin{align*} \phat_{C+1}^\star=0. \end{align*}

    With no reservation, the objective becomes
    \begin{align*}
        \W{r}=\sum_{j=1}^{C}p_j\log(\costmult\phat_j)=\log(\costmult)+\sum_{j=1}^{C}p_j\log\phat_j.
    \end{align*}
    Thus, maximizing \(\W{r}\) is equivalent to maximizing \(\sum_{j=1}^{C}p_j\log\phat_j\) over the probability simplex. By Gibbs' inequality, the maximum is attained at
    \begin{align*}
        \phat_j^\star=p_j, \qquad j=1,\ldots,C.
    \end{align*}
    Hence, for \(\costmult>C\), every optimal solution has no reservation and predicts according to the posterior distribution.

    \paragraph{Case 3: \(\costmult=1\).}
    For every feasible \(\phat\) and every \(j\in\{1,\ldots,C\}\),
    \begin{align*}
        \phat_j+\phat_{C+1}\le1.
    \end{align*}
    Therefore,
    \begin{align*}
        \W{r}\le0,
    \end{align*}
    and full reservation attains \(\W{r}=0\). Hence, full reservation is optimal. However, the optimizer need not be unique. For example, suppose the posterior is concentrated on a single class \(i\), so that
    \begin{align*}
        p_i=1, \qquad p_j=0 \quad \text{for } j\neq i.
    \end{align*}
    Then every feasible solution satisfying
    \begin{align*}
        \phat_i+\phat_{C+1}=1, \qquad \phat_j=0 \quad \text{for } j\neq i,
    \end{align*}
    achieves
    \begin{align*}
        \W{r}=\log(\phat_i+\phat_{C+1})=0.
    \end{align*}
    Thus, distinct prediction--reservation allocations may be equally optimal when \(\costmult=1\).

    \paragraph{Case 4: \(\costmult=C\).}
    The optimizer need not be unique at this boundary either. Consider the uniform posterior
    \begin{align*}
        p_j=\frac{1}{C}, \qquad j=1,\ldots,C.
    \end{align*}
    For any \(\alpha\in[0,1]\), define
    \begin{align*}
        \phat_j^{(\alpha)}=\frac{1-\alpha}{C}, \qquad j=1,\ldots,C, \qquad \phat_{C+1}^{(\alpha)}=\alpha.
    \end{align*}
    Then
    \begin{align*}
        C\phat_j^{(\alpha)}+\phat_{C+1}^{(\alpha)}=1, \qquad j=1,\ldots,C,
    \end{align*}
    and consequently
    \begin{align*}
        \W{r}=\sum_{j=1}^{C}\frac{1}{C}\log(1)=0.
    \end{align*}
    To see that this value is optimal, Jensen's inequality gives
    \begin{align*}
        \W{r} & \le                                                                          \\
              & \log\left(\frac{1}{C}\sum_{j=1}^{C}\left(C\phat_j+\phat_{C+1}\right)\right)= \\
              & \log(1)=0.
    \end{align*}
    Thus, every member of this family for different values of \(\alpha\) is optimal, showing that distinct prediction--reservation allocations may also be equally optimal when \(\costmult=C\).

    Therefore, \(\costmult<1\) corresponds to the full-reservation regime, while \(\costmult>C\) corresponds to the no-reservation prediction regime. The boundary values \(\costmult=1\) and \(\costmult=C\) may admit multiple optimal prediction--reservation allocations. The cost-multiplier mixed-regime, which is the most interesting case, is \(1<\costmult<C\).
\end{proof}

\subsection{Additional numerical results}
\label{subsec:additional_numerical_results}
In this section, we present additional numerical results, which complement the results already presented in the main body of the paper.
\begin{figure}[!h]
    \includegraphics[width=.48\textwidth]{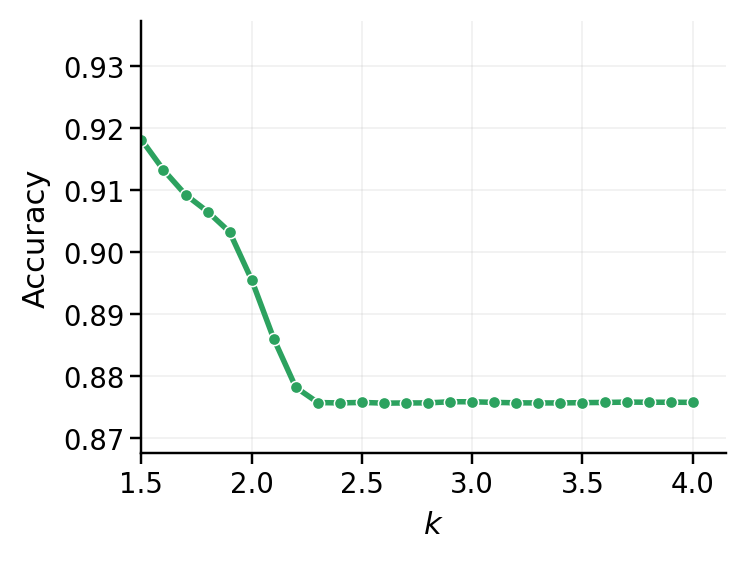}
    \caption{Synthetic benchmark data -- Ablation study over the variable \(k\) for the higher posterior overlap case for the benchmark in \Cref{subsec:synthetic-benchmark}.}
    \label{fig:ablation-k-synthetic}
\end{figure}
\subsection{Benchmark generation}
\label{subsec:benchmark_gen}
\input{tex/tex_files/sections/synthetic_data_genereation.tex}

%% file: tex/tex_files/sections/synthetic_data_genereation.tex
\captionof{algorithm}{Synthetic Data Generation with Target Bayes Error}
\label{alg:generator}

\begin{algorithmic}[1]

\Require Number of samples $N$; feature means $\bm{\mu}$ and scales
$\bm{\sigma}$; correlation matrix $\mathbf{R}$; causal defect map
$\mathcal{E}$; class priors $\bm{\pi}$; target Bayes error
$\varepsilon^\star$
\Ensure Synthetic labeled dataset $\mathcal{D}$

\State Sample correlated process noise
\[
    \mathbf{Z} \sim \mathcal{N}(\mathbf{0},\mathbf{R})
\]

\State Generate process features
\[
    \mathbf{X}
    = \bm{\mu} + \bm{\sigma}\odot\mathbf{Z}
\]

\State Normalize each feature relative to its specification limits
to obtain deviations $\bm{\Delta}$

\State Compute class scores $s_k(\bm{\Delta};\mathcal{E})$
from the causal defect map

\State Initialize gain interval $[g_{\min},g_{\max}]$

\Repeat
    \State $g \gets (g_{\min}+g_{\max})/2$

    \State Calibrate class offsets $b_k$ so that the mean posterior
    probabilities match the desired priors $\pi_k$

    \State Compute the class posterior
    \[
        p_{ik}
        =
        P(Y=k\mid\mathbf{x}_i)
        =
        \operatorname{softmax}_k
        \left(g\,s_k(\bm{\Delta}_i)+b_k\right)
    \]

    \State Compute the resulting Bayes error
    \[
        \varepsilon_B
        =
        \frac{1}{N}
        \sum_{i=1}^{N}
        \left(1-\max_k p_{ik}\right)
    \]

    \If{$\varepsilon_B > \varepsilon^\star$}
        \State $g_{\min} \gets g$
    \Else
        \State $g_{\max} \gets g$
    \EndIf
\Until{$\varepsilon_B \approx \varepsilon^\star$}

\For{$i=1,\ldots,N$}
    \State Sample
    $y_i \sim \operatorname{Categorical}(p_{i1},\ldots,p_{iK})$
\EndFor

\State $\mathcal{D}\gets\{(\mathbf{x}_i,y_i)\}_{i=1}^{N}$

\State \Return $\mathcal{D}$

\end{algorithmic}

%% file: tex/bibliography.bib
@book{boyd2004convex,
  title     = {Convex Optimization},
  author    = {Boyd, Stephen and Vandenberghe, Lieven},
  year      = {2004},
  publisher = {Cambridge University Press},
  address   = {Cambridge, UK},
  isbn      = {978-0521833783}
}

@article{chow1970optimum,
  author   = {Chow, C.},
  journal  = {IEEE Transactions on Information Theory},
  title    = {On optimum recognition error and reject tradeoff},
  year     = {1970},
  volume   = {16},
  number   = {1},
  pages    = {41-46},
  doi      = {10.1109/TIT.1970.1054406}
}

@inproceedings{liu2019deep,
  author    = {Liu, Ziyin and Wang, Zhikang and Liang, Paul Pu and Salakhutdinov, Russ R and Morency, Louis-Philippe and Ueda, Masahito},
  booktitle = {Advances in Neural Information Processing Systems},
  editor    = {H. Wallach and H. Larochelle and A. Beygelzimer and F. d\textquotesingle Alch\'{e}-Buc and E. Fox and R. Garnett},
  pages     = {},
  publisher = {Curran Associates, Inc.},
  title     = {Deep Gamblers: Learning to Abstain with Portfolio Theory},
  volume    = {32},
  year      = {2019}
}

@article{markowitz1952portfolio,
  issn      = {00221082, 15406261},
  author    = {Harry Markowitz},
  journal   = {The Journal of Finance},
  number    = {1},
  pages     = {77--91},
  publisher = {[American Finance Association, Wiley]},
  title     = {Portfolio Selection},
  urldate   = {2026-06-23},
  volume    = {7},
  year      = {1952}
}

@article{shenoy2026uncertainty,
  author   = {Shenoy, Mukund and Ameri, Farhad},
  journal  = {IEEE Transactions on Semiconductor Manufacturing},
  title    = {Uncertainty-Aware Neurosymbolic Root-Cause Analysis for Surface-Mount Assembly},
  year     = {2026},
  volume   = {},
  number   = {},
  pages    = {1-1},
  doi      = {10.1109/TSM.2026.3673999}
}

@article{lee2017convolutional,
  author  = {Lee, K. B. and Cheon, S. and Kim, C. O.},
  title   = {A Convolutional Neural Network for Fault Classification and Diagnosis in Semiconductor Manufacturing Processes},
  journal = {IEEE Transactions on Semiconductor Manufacturing},
  volume  = {30},
  number  = {2},
  pages   = {135--142},
  year    = {2017},
  doi     = {10.1109/TSM.2017.2676245}
}

@article{khader2018stencil,
  author  = {Khader, N. and Yoon, S. W.},
  title   = {Stencil Printing Process Optimization to Control Solder Paste Volume Transfer Efficiency},
  journal = {IEEE Transactions on Components, Packaging and Manufacturing Technology},
  volume  = {8},
  number  = {9},
  pages   = {1686--1694},
  year    = {2018},
  doi     = {10.1109/TCPMT.2018.2830391}
}

@article{lu2019realtime,
  author  = {Lu, H. and Wang, H. and Yoon, S. W. and Won, D.},
  title   = {Real-Time Stencil Printing Optimization Using a Hybrid Multi-Layer Online Sequential Extreme Learning and Evolutionary Search Approach},
  journal = {IEEE Transactions on Components, Packaging and Manufacturing Technology},
  volume  = {9},
  number  = {12},
  pages   = {2490--2501},
  year    = {2019},
  doi     = {10.1109/TCPMT.2019.2934487}
}

@article{cao2024xscan,
  author  = {Cao, N. and Farrag, A. and Won, D. and Yoon, S. W.},
  title   = {{XSCAN}: Explainable Solder Joint Defect Probability Prediction through Solder Paste Printing Status with Imbalanced Data},
  journal = {Journal of Manufacturing Systems},
  volume  = {77},
  pages   = {212--227},
  year    = {2024},
  doi     = {10.1016/j.jmsy.2024.09.009}
}

@inproceedings{kim2024deep,
  author    = {Kim, Heeyeon and Kim, Ki-Hyung},
  booktitle = {2024 International Conference on Artificial Intelligence in Information and Communication (ICAIIC)},
  title     = {Deep Learning-Based SBOM Defect Detection for Medical Devices},
  year      = {2024},
  volume    = {},
  number    = {},
  pages     = {047-051},
  doi       = {10.1109/ICAIIC60209.2024.10463483}
}

@misc{huang2026unsuperviseddefectdetectionsurgical,
  title         = {Unsupervised Defect Detection for Surgical Instruments},
  author        = {Joseph Huang and Yichi Zhang and Jingxi Yu and Wei Chen and Seunghyun Hwang and Qiang Qiu and Amy R. Reibman and Edward J. Delp and Fengqing Zhu},
  year          = {2026},
  eprint        = {2509.21561},
  archiveprefix = {arXiv},
  primaryclass  = {cs.CV},
  url           = {https://arxiv.org/abs/2509.21561}
}

@article{bonatti2025expert,
  author  = {Bonatti, A. F. and Meringolo, F. D. and Tubertini, I. and Lavecchia, C. E. and Favaro, A. and Vozzi, G. and De Maria, C.},
  title   = {Expert System for Online Defect Detection in Medical Devices Produced by Electron Beam Melting Using Layer-by-Layer Optical Images},
  journal = {3D Printing and Additive Manufacturing},
  year    = {2025},
  volume  = {12},
  number  = {1},
  pages   = {36--47},
  doi     = {10.1089/3dp.2023.0222},
  pmid    = {40151675},
  pmcid   = {PMC11937765}
}

@online{protookn,
  author  = {{Proto-OKN Collaboration}},
  title   = {{Proto-OKN: Prototype Open Knowledge Network}},
  url     = {https://www.proto-okn.net/},
  urldate = {2026-08-18}
}

@online{oknregistry,
  author  = {{Open Knowledge Network}},
  title   = {{Open Knowledge Network (OKN) Graph Registry}},
  url     = {https://registry.okn.us/},
  urldate = {2026-08-18}
}

@misc{fda_maude,
  author       = {{U.S. Food and Drug Administration}},
  title        = {{Manufacturer and User Facility Device Experience (MAUDE) Database}},
  howpublished = {\url{https://www.accessdata.fda.gov/scripts/cdrh/cfdocs/cfmaude/search.cfm}},
  note         = {Accessed: 2026-08-18},
  year         = {2026}
}

@misc{fda_510k,
  author       = {{U.S. Food and Drug Administration}},
  title        = {{510(k) Premarket Notification Database}},
  howpublished = {\url{https://www.accessdata.fda.gov/scripts/cdrh/cfdocs/cfpmn/pmn.cfm}},
  note         = {Accessed: 2026-08-18},
  year         = {2026}
}

@article{Chow1970TIT,
  author    = {C. K. Chow},
  title     = {On optimum recognition error and reject tradeoff},
  journal   = {{IEEE} Trans. Inf. Theory},
  volume    = {16},
  number    = {1},
  pages     = {41--46},
  year      = {1970},
  url       = {https://doi.org/10.1109/TIT.1970.1054406},
  doi       = {10.1109/TIT.1970.1054406},
  bibsource = {dblp computer science bibliography, https://dblp.org}
}

@inproceedings{CorbiereTBCP2019NeurIPS,
  author    = {Charles Corbi{\`{e}}re and
               Nicolas Thome and
               Avner Bar{-}Hen and
               Matthieu Cord and
               Patrick P{\'{e}}rez},
  editor    = {Hanna M. Wallach and
               Hugo Larochelle and
               Alina Beygelzimer and
               Florence d'Alch{\'{e}}{-}Buc and
               Emily B. Fox and
               Roman Garnett},
  title     = {Addressing Failure Prediction by Learning Model Confidence},
  booktitle = {Advances in Neural Information Processing Systems 32: Annual Conference
               on Neural Information Processing Systems 2019, NeurIPS 2019, December
               8-14, 2019, Vancouver, BC, Canada},
  pages     = {2898--2909},
  year      = {2019},
  url       = {https://proceedings.neurips.cc/paper/2019/hash/757f843a169cc678064d9530d12a1881-Abstract.html},
  bibsource = {dblp computer science bibliography, https://dblp.org}
}

@article{FengAHA2022CORR,
  author     = {Leo Feng and
                Mohamed Osama Ahmed and
                Hossein Hajimirsadeghi and
                Amir H. Abdi},
  title      = {Stop Overcomplicating Selective Classification: Use Max-Logit},
  journal    = {CoRR},
  volume     = {abs/2206.09034},
  year       = {2022},
  url        = {https://doi.org/10.48550/arXiv.2206.09034},
  doi        = {10.48550/arXiv.2206.09034},
  eprinttype = {arXiv},
  eprint     = {2206.09034},
  bibsource  = {dblp computer science bibliography, https://dblp.org}
}

@article{FischJB2022CORR,
  author     = {Adam Fisch and
                Tommi Jaakkola and
                Regina Barzilay},
  title      = {Calibrated Selective Classification},
  journal    = {CoRR},
  volume     = {abs/2208.12084},
  year       = {2022},
  url        = {https://doi.org/10.48550/arXiv.2208.12084},
  doi        = {10.48550/arXiv.2208.12084},
  eprinttype = {arXiv},
  eprint     = {2208.12084},
  bibsource  = {dblp computer science bibliography, https://dblp.org}
}

@article{Flores1958TEC,
  author    = {Ivan Flores},
  title     = {An Optimum Character Recognition System Using Decision Functions},
  journal   = {{IRE} Trans. Electron. Comput.},
  volume    = {7},
  number    = {2},
  pages     = {180},
  year      = {1958},
  url       = {https://doi.org/10.1109/TEC.1958.5222530},
  doi       = {10.1109/TEC.1958.5222530},
  bibsource = {dblp computer science bibliography, https://dblp.org}
}

@inproceedings{GeifmanE2019ICML,
  author    = {Yonatan Geifman and
               Ran El{-}Yaniv},
  editor    = {Kamalika Chaudhuri and
               Ruslan Salakhutdinov},
  title     = {SelectiveNet: {A} Deep Neural Network with an Integrated Reject Option},
  booktitle = {Proceedings of the 36th International Conference on Machine Learning,
               {ICML} 2019, 9-15 June 2019, Long Beach, California, {USA}},
  series    = {Proceedings of Machine Learning Research},
  volume    = {97},
  pages     = {2151--2159},
  publisher = {{PMLR}},
  year      = {2019},
  url       = {http://proceedings.mlr.press/v97/geifman19a.html},
  bibsource = {dblp computer science bibliography, https://dblp.org}
}

@inproceedings{GraneseRGPP2021NeurIPS,
  author    = {Federica Granese and
               Marco Romanelli and
               Daniele Gorla and
               Catuscia Palamidessi and
               Pablo Piantanida},
  editor    = {Marc'Aurelio Ranzato and
               Alina Beygelzimer and
               Yann N. Dauphin and
               Percy Liang and
               Jennifer Wortman Vaughan},
  title     = {{DOCTOR:} {A} Simple Method for Detecting Misclassification Errors},
  booktitle = {Advances in Neural Information Processing Systems 34: Annual Conference
               on Neural Information Processing Systems 2021, NeurIPS 2021, December
               6-14, 2021, virtual},
  pages     = {5669--5681},
  year      = {2021},
  url       = {https://proceedings.neurips.cc/paper/2021/hash/2cb6b10338a7fc4117a80da24b582060-Abstract.html},
  bibsource = {dblp computer science bibliography, https://dblp.org}
}

@inproceedings{GuoPSW2017ICML,
  author    = {Chuan Guo and
               Geoff Pleiss and
               Yu Sun and
               Kilian Q. Weinberger},
  editor    = {Doina Precup and
               Yee Whye Teh},
  title     = {On Calibration of Modern Neural Networks},
  booktitle = {Proceedings of the 34th International Conference on Machine Learning,
               {ICML} 2017, Sydney, NSW, Australia, 6-11 August 2017},
  series    = {Proceedings of Machine Learning Research},
  volume    = {70},
  pages     = {1321--1330},
  publisher = {{PMLR}},
  year      = {2017},
  url       = {http://proceedings.mlr.press/v70/guo17a.html},
  bibsource = {dblp computer science bibliography, https://dblp.org}
}

@inproceedings{HuangZZ2020NeurIPS,
  author    = {Lang Huang and
               Chao Zhang and
               Hongyang Zhang},
  editor    = {Hugo Larochelle and
               Marc'Aurelio Ranzato and
               Raia Hadsell and
               Maria{-}Florina Balcan and
               Hsuan{-}Tien Lin},
  title     = {Self-Adaptive Training: beyond Empirical Risk Minimization},
  booktitle = {Advances in Neural Information Processing Systems 33: Annual Conference
               on Neural Information Processing Systems 2020, NeurIPS 2020, December
               6-12, 2020, virtual},
  year      = {2020},
  url       = {https://proceedings.neurips.cc/paper/2020/hash/e0ab531ec312161511493b002f9be2ee-Abstract.html},
  bibsource = {dblp computer science bibliography, https://dblp.org}
}

@inproceedings{PudilNBK1992ICPR,
  author    = {Pavel Pudil and
               Jana Novovicov{\'{a}} and
               Svatopluk Bl{\'{a}}ha and
               Josef Kittler},
  title     = {Multistage pattern recognition with reject option},
  booktitle = {11th {IAPR} International Conference on Pattern Recognition, {ICPR}
               1992. Conference {B:} Pattern Recognition Methodology and Systems,
               The Hague, Netherlands, August 30-September 3, 1992},
  pages     = {92--95},
  publisher = {{IEEE}},
  year      = {1992},
  url       = {https://doi.org/10.1109/ICPR.1992.201729},
  doi       = {10.1109/ICPR.1992.201729},
  bibsource = {dblp computer science bibliography, https://dblp.org}
}

@inproceedings{DBLP:conf/cvpr/Mukhoti0ATG23,
  author    = {Jishnu Mukhoti and
               Andreas Kirsch and
               Joost van Amersfoort and
               Philip H. S. Torr and
               Yarin Gal},
  title     = {Deep Deterministic Uncertainty: {A} New Simple Baseline},
  booktitle = {{IEEE/CVF} Conference on Computer Vision and Pattern Recognition,
               {CVPR} 2023, Vancouver, BC, Canada, June 17-24, 2023},
  pages     = {24384--24394},
  publisher = {{IEEE}},
  year      = {2023},
  url       = {https://doi.org/10.1109/CVPR52729.2023.02336},
  doi       = {10.1109/CVPR52729.2023.02336},
  bibsource = {dblp computer science bibliography, https://dblp.org}
}

@inproceedings{paszke2019pytorch,
  title     = {{PyTorch}: An Imperative Style, High-Performance Deep Learning Library},
  author    = {Paszke, Adam and Gross, Sam and Massa, Francisco and Lerer, Adam and Bradbury, James and Chanan, Gregory and Killeen, Trevor and Lin, Zeming and Gimelshein, Natalia and Antiga, Luca and Desmaison, Alban and K{\"o}pf, Andreas and Yang, Edward and DeVito, Zachary and Raison, Martin and Tejani, Alykhan and Chilamkurthy, Sasank and Steiner, Benoit and Fang, Lu and Bai, Junjie and Chintala, Soumith},
  booktitle = {Advances in Neural Information Processing Systems},
  volume    = {32},
  year      = {2019}
}

@article{hullermeier2021aleatoric,
  author  = {H{\"u}llermeier, Eyke and Waegeman, Willem},
  title   = {Aleatoric and epistemic uncertainty in machine learning: an introduction to concepts and methods},
  journal = {Machine Learning},
  volume  = {110},
  pages   = {457--506},
  year    = {2021},
  doi     = {10.1007/s10994-021-05946-3}
}

@article{kompa2021second,
  author  = {Kompa, Benjamin and Snoek, Jasper and Beam, Andrew L.},
  title   = {Second opinion needed: communicating uncertainty in medical machine learning},
  journal = {npj Digital Medicine},
  volume  = {4},
  number  = {4},
  year    = {2021},
  doi     = {10.1038/s41746-020-00367-3}
}

@article{begoli2019need,
  author  = {Begoli, Edmon and Bhattacharya, Tanmoy and Kusnezov, Dimitri},
  title   = {The need for uncertainty quantification in machine-assisted medical decision making},
  journal = {Nature Machine Intelligence},
  volume  = {1},
  pages   = {20--23},
  year    = {2019},
  doi     = {10.1038/s42256-018-0004-1}
}

@article{nicholson2020constructing,
  author  = {Nicholson, David N. and Greene, Casey S.},
  title   = {Constructing knowledge graphs and their biomedical applications},
  journal = {Computational and Structural Biotechnology Journal},
  volume  = {18},
  pages   = {1414--1428},
  year    = {2020},
  doi     = {10.1016/j.csbj.2020.05.017}
}

@article{chung2020signal,
  author  = {Chung, Grace and Etter, Katherine and Yoo, Andrew},
  title   = {Medical device active surveillance of spontaneous reports: A literature review of signal detection methods},
  journal = {Pharmacoepidemiology and Drug Safety},
  volume  = {29},
  number  = {4},
  pages   = {369--379},
  year    = {2020},
  doi     = {10.1002/pds.4980}
}

@article{mao2021nlp,
  author  = {Mao, Jialin and Sedrakyan, Art and Sun, Tianyi and Guiahi, Maryam and Chudnoff, Scott and Kinard, Madris and Johnson, Stephen B.},
  title   = {Assessing adverse event reports of hysteroscopic sterilization device removal using natural language processing},
  journal = {Pharmacoepidemiology and Drug Safety},
  volume  = {31},
  number  = {4},
  pages   = {442--451},
  year    = {2021},
  doi     = {10.1002/pds.5402}
}

@article{sohoni2026llm,
  author  = {Sohoni, Nishan and Sohoni, Nimit S. and Sutherland, Ryan A. and Sundaresan, Vinaik M. and Smani, Shayan and Ananth, Prasanna and Onofrey, John A. and Aneja, Sanjay and Miszczyk, Marcin and Lee, Ho-Joon and Olivieri, Julia E. and Leapman, Michael S.},
  title   = {Automated classification of adverse events after hydrogel perirectal spacer insertion for prostate cancer using large language models},
  journal = {Urology},
  volume  = {209},
  pages   = {18--24},
  year    = {2026},
  doi     = {10.1016/j.urology.2026.01.008}
}

@article{slivinskis2025recall,
  author  = {Barbosa Slivinskis, Victor and Agi Maluli, Isabela and Seth Broder, Joshua},
  title   = {A machine learning algorithm to predict medical device recall by the {Food and Drug Administration}},
  journal = {Western Journal of Emergency Medicine},
  volume  = {26},
  number  = {1},
  pages   = {161--170},
  year    = {2025},
  doi     = {10.5811/westjem.21238}
}

@article{xu2015changepoint,
  author  = {Xu, Zhiheng and Kass-Hout, Taha and Anderson-Smits, Colin and Gray, Gerry},
  title   = {Signal detection using change point analysis in postmarket surveillance},
  journal = {Pharmacoepidemiology and Drug Safety},
  volume  = {24},
  number  = {6},
  pages   = {663--668},
  year    = {2015},
  doi     = {10.1002/pds.3783}
}

@inproceedings{lakshminarayanan2017ensembles,
  author    = {Lakshminarayanan, Balaji and Pritzel, Alexander and Blundell, Charles},
  title     = {Simple and scalable predictive uncertainty estimation using deep ensembles},
  booktitle = {Advances in Neural Information Processing Systems},
  volume    = {30},
  year      = {2017}
}

@inproceedings{kendall2017uncertainties,
  author    = {Kendall, Alex and Gal, Yarin},
  title     = {What uncertainties do we need in {Bayesian} deep learning for computer vision?},
  booktitle = {Advances in Neural Information Processing Systems},
  volume    = {30},
  year      = {2017}
}

@book{vovk2005algorithmic,
  author    = {Vovk, Vladimir and Gammerman, Alexander and Shafer, Glenn},
  title     = {Algorithmic Learning in a Random World},
  publisher = {Springer},
  address   = {New York, NY},
  year      = {2005},
  doi       = {10.1007/b106715},
  isbn      = {978-0-387-00152-4}
}

@article{angelopoulos2023conformal,
  author  = {Angelopoulos, Anastasios N. and Bates, Stephen},
  title   = {Conformal Prediction: A Gentle Introduction},
  journal = {Foundations and Trends in Machine Learning},
  volume  = {16},
  number  = {4},
  pages   = {494--591},
  year    = {2023},
  doi     = {10.1561/2200000101}
}

@article{bates2021distribution,
  author    = {Bates, Stephen and Angelopoulos, Anastasios and
               Lei, Lihua and Malik, Jitendra and Jordan, Michael I.},
  title     = {Distribution-Free, Risk-Controlling Prediction Sets},
  journal   = {Journal of the ACM},
  volume    = {68},
  number    = {6},
  articleno = {43},
  numpages  = {34},
  year      = {2021},
  doi       = {10.1145/3478535}
}

@article{elyaniv2010foundations,
  author  = {El-Yaniv, Ran and Wiener, Yair},
  title   = {On the Foundations of Noise-Free Selective Classification},
  journal = {Journal of Machine Learning Research},
  volume  = {11},
  number  = {53},
  pages   = {1605--1641},
  year    = {2010},
  url     = {https://jmlr.org/papers/v11/el-yaniv10a.html}
}

@book{shalev2014understanding,
  author    = {Shalev-Shwartz, Shai and Ben-David, Shai},
  ee        = {http://www.cambridge.org/de/academic/subjects/computer-science/pattern-recognition-and-machine-learning/understanding-machine-learning-theory-algorithms},
  isbn      = {978-1-10-705713-5},
  pages     = {I-XVI, 1-397},
  publisher = {Cambridge University Press},
  title     = {Understanding Machine Learning - From Theory to Algorithms.},
  year      = 2014
}

@book{bishop2006pattern,
  author      = {Bishop, C.M.},
  description = {Pattern Recognition and Machine Learning - Christopher M. Bishop - Google Books},
  isbn        = {9780387310732},
  lccn        = {2006922522},
  publisher   = {Springer},
  series      = {Information Science and Statistics},
  title       = {Pattern Recognition and Machine Learning},
  url         = {https://books.google.co.in/books?id=kTNoQgAACAAJ},
  year        = 2006
}

@article{goodfellow2016deep,
  author    = {Heaton, Jeff},
  ee        = {https://www.wikidata.org/entity/Q130466682},
  journal   = {Genet. Program. Evolvable Mach.},
  number    = {1-2},
  pages     = {305-307},
  title     = {Ian Goodfellow, Yoshua Bengio, and Aaron Courville: Deep learning - The MIT Press, 2016, 800 pp, ISBN: 0262035618.},
  url       = {http://dblp.uni-trier.de/db/journals/gpem/gpem19.html#Heaton18},
  volume    = 19,
  year      = 2018
}

@article{Hornik1991,
  title   = {Approximation capabilities of multilayer feedforward networks},
  author  = {Kurt Hornik},
  journal = {Neural Networks},
  volume  = {4},
  number  = {2},
  pages   = {251--257},
  year    = {1991},
  issn    = {0893-6080},
  doi     = {10.1016/0893-6080(91)90009-T},
  url     = {https://www.sciencedirect.com/science/article/pii/089360809190009T}
}

@misc{kingma2014method,
  author      = {Kingma, Diederik P. and Ba, Jimmy},
  description = {Adam: A Method for Stochastic Optimization},
  note        = {cite arxiv:1412.6980Comment: Published as a conference paper at the 3rd International Conference  for Learning Representations, San Diego, 2015},
  title       = {Adam: A Method for Stochastic Optimization},
  url         = {http://arxiv.org/abs/1412.6980},
  year        = 2014
}
